%% file: main.tex
\documentclass[a4paper,fleqn]{cas-dc}

\usepackage[numbers]{natbib}
\usepackage{amsmath,amsthm,amssymb,amsfonts}
\usepackage{cases}
\usepackage{algorithm}
\usepackage{algorithmicx}
\usepackage{algpseudocode}
\usepackage{array}
\usepackage{booktabs}
\usepackage{tabularray}
\usepackage{multirow}
\usepackage{siunitx}
\usepackage{subfig}
\usepackage{enumerate}
\usepackage{textcomp}
\usepackage{url}
\usepackage{pbox}
\usepackage{bm}
\usepackage{xcolor}
\usepackage{colortbl}
\usepackage{soul}
\usepackage{pifont}
\usepackage{alltt}
\usepackage{float}
\usepackage[section]{placeins}

\newtheorem{assumption}{Assumption}
\newtheorem{theorem}{Theorem}
\newtheorem{problem}{Problem}
\newtheorem{lemma}{Lemma}

\newtheorem{corollary}{Corollary}
\newtheorem{remark}{Remark}

\definecolor{limegreen}{rgb}{0.2, 0.8, 0.2}
\definecolor{forestgreen}{rgb}{0.13, 0.55, 0.13}
\definecolor{greenhtml}{rgb}{0.0, 0.5, 0.0}

\begin{document}
\makeatletter
\def\ps@first{\let\@oddhead\@empty\let\@evenhead\@empty\let\@oddfoot\@empty\let\@evenfoot\@empty}
\def\ps@cas{\let\@oddhead\@empty\let\@evenhead\@empty\let\@oddfoot\@empty\let\@evenfoot\@empty}
\pagestyle{empty}
\makeatother
\ExplSyntaxOn
\bool_gset_true:N \g_stm_nologo_bool
\ExplSyntaxOff
\let\WriteBookmarks\relax
\def\floatpagepagefraction{1}
\def\textpagefraction{.001}

\shorttitle{Decentralized Pose Graph Riemannian Optimization}
\shortauthors{Y. Zhao et~al.}

\title[mode=title]{Decentralized Pose Graph Riemannian Optimization for Object-based Multi-Robot SLAM}

\author[1]{Yixian Zhao}[orcid=0009-0001-1652-7574]
\ead{12132038@zju.edu.cn}

\author[3]{Yan Huang}
\ead{yahuang@kth.se}

\author[2]{Yang Xu}
\ead{xuyang_robot@tongji.edu.cn}

\author[1]{Liang Li}
\ead{liang.li@zju.edu.cn}

\author[1]{Jinming Xu}
\cormark[1]
\ead{jimmyxu@zju.edu.cn}

\affiliation[1]{
  organization={College of Control Science and Engineering, Zhejiang University},
  city={Hangzhou},
  postcode={310027},
  country={China}
}

\affiliation[2]{
  organization={Shanghai Research Institute for Intelligent Autonomous Systems, Tongji University},
  city={Shanghai},
  postcode={201210},
  country={China}
}

\affiliation[3]{
  organization={Division of Decision and Control Systems, School of EECS, KTH Royal Institute of Technology},
  city={Stockholm},
  postcode={SE-100 44},
  country={Sweden}
}

\cortext[cor1]{Corresponding author.}

\begin{abstract}
Pose graph optimization (PGO) is a key back-end component for state estimation in networked multi-robot simultaneous localization and mapping (SLAM). In object-based multi-robot SLAM, the problem becomes more tightly coupled because robots must jointly estimate both their trajectories and the poses of persistent objects observed by multiple agents. Existing decentralized solutions often assume that the communication graph closely matches the physical interaction topology, which is restrictive in realistic deployments where communication is sparse, intermittent, or time-varying. This paper presents a fully decentralized Riemannian optimization framework for object-based multi-robot PGO that decouples the coupled estimation problem via a consensus mechanism, enabling flexible communication topologies. To improve convergence under limited communication budgets, we further develop a distributed approximate-Newton scheme that exploits local second-order information while operating directly on the SE(d) manifold to preserve geometric consistency, and we establish the convergence to Riemannian first-order stationary points and provide a local condition-number analysis explaining the benefit of approximate second-order information over first-order Riemannian descent. The resulting method reduces iteration count and communication overhead without sacrificing estimation accuracy. Extensive evaluations on public benchmarks, large-scale simulations, and real-world multi-robot experiments demonstrate improved accuracy, runtime efficiency, scalability across network topologies, and robustness to communication failures.
\end{abstract}

\begin{keywords}
Simultaneous localization and mapping \sep Distributed optimization \sep Multi-robot systems \sep Pose graph optimization \sep Riemannian manifold
\end{keywords}

\maketitle

\input{article_body}

\section*{Declaration of Competing Interest}
The authors declare that they have no known competing financial interests or personal relationships that could have appeared to influence the work reported in this paper.

\section*{Data Availability}
Data and code will be made available on reasonable request.

\printcredits

\bibliographystyle{cas-model2-names}
\bibliography{ref}

\input{appendix}

\end{document}

%% file: article_body.tex
\section{Introduction}

Multi-robot systems are increasingly deployed in safety-critical and resource-constrained missions such as search-and-rescue, subterranean exploration, and large-scale inspection. In these applications, the closed-loop performance of autonomy pipelines (planning, control, and coordination) critically depends on reliable state estimation\cite{kimera_multi,schmuck2021covins,xuyang2025,slam_tase_2025}. Multi-robot simultaneous localization and mapping (SLAM) addresses this need by enabling robots to build a consistent map while estimating their trajectories, and pose graph optimization (PGO) serves as a core back-end module that improves global consistency by fusing relative measurements across time and agents \cite{halsted2021survey,ott2023fusing}.

\begin{figure}
    \centering
    \label{fig:dingo}
    \subfloat[A multiple Dingo robot system]{
        \label{fig:dingo_system}
        \includegraphics[width=8cm]{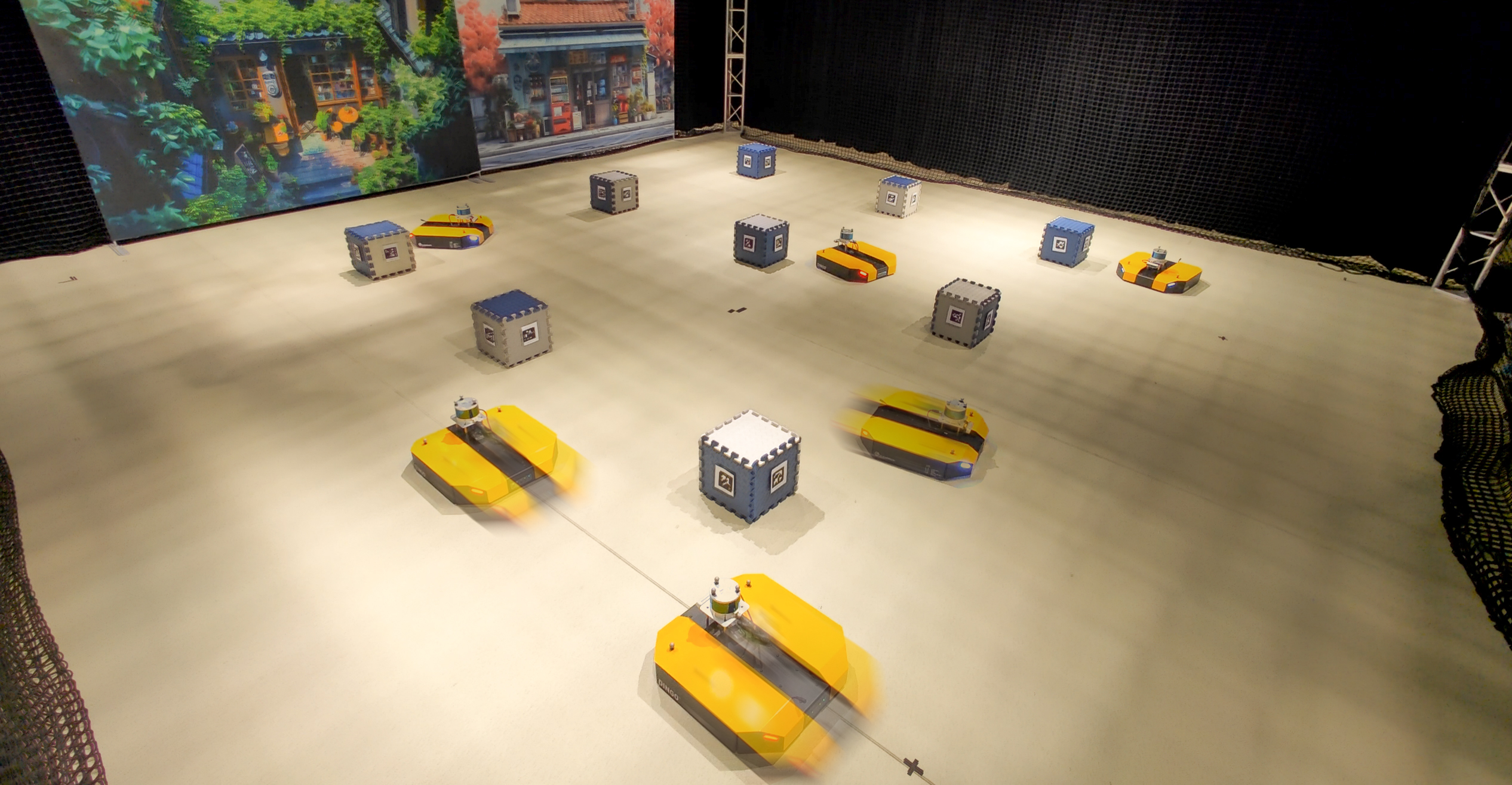}}
    \newline
    \subfloat[Initial poses]{
        \label{fig:dingo_initial}
        \includegraphics[width=4cm]{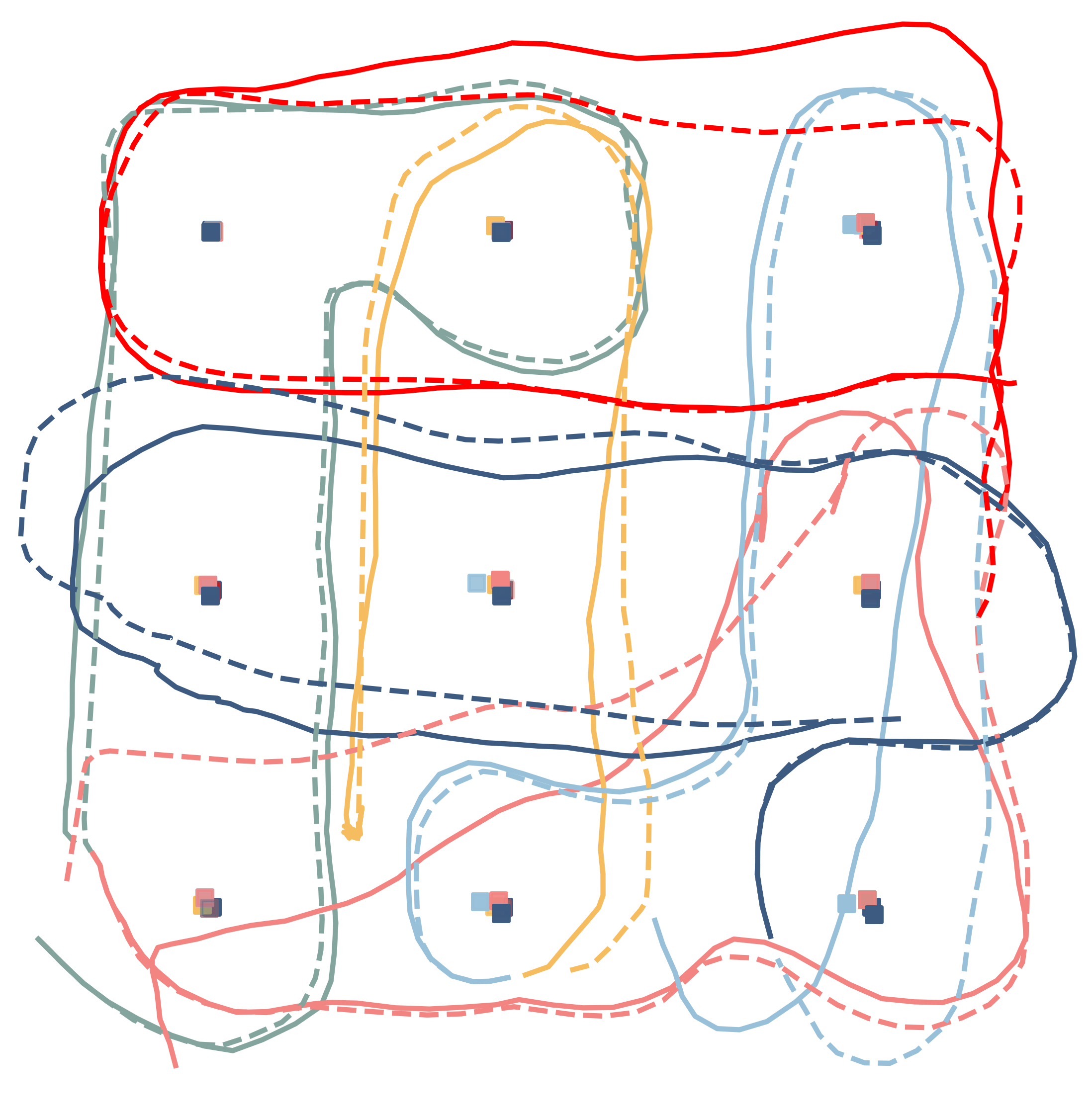}}
    \subfloat[Optimized poses]{
        \label{fig:dingo_optimized}
        \includegraphics[width=4cm]{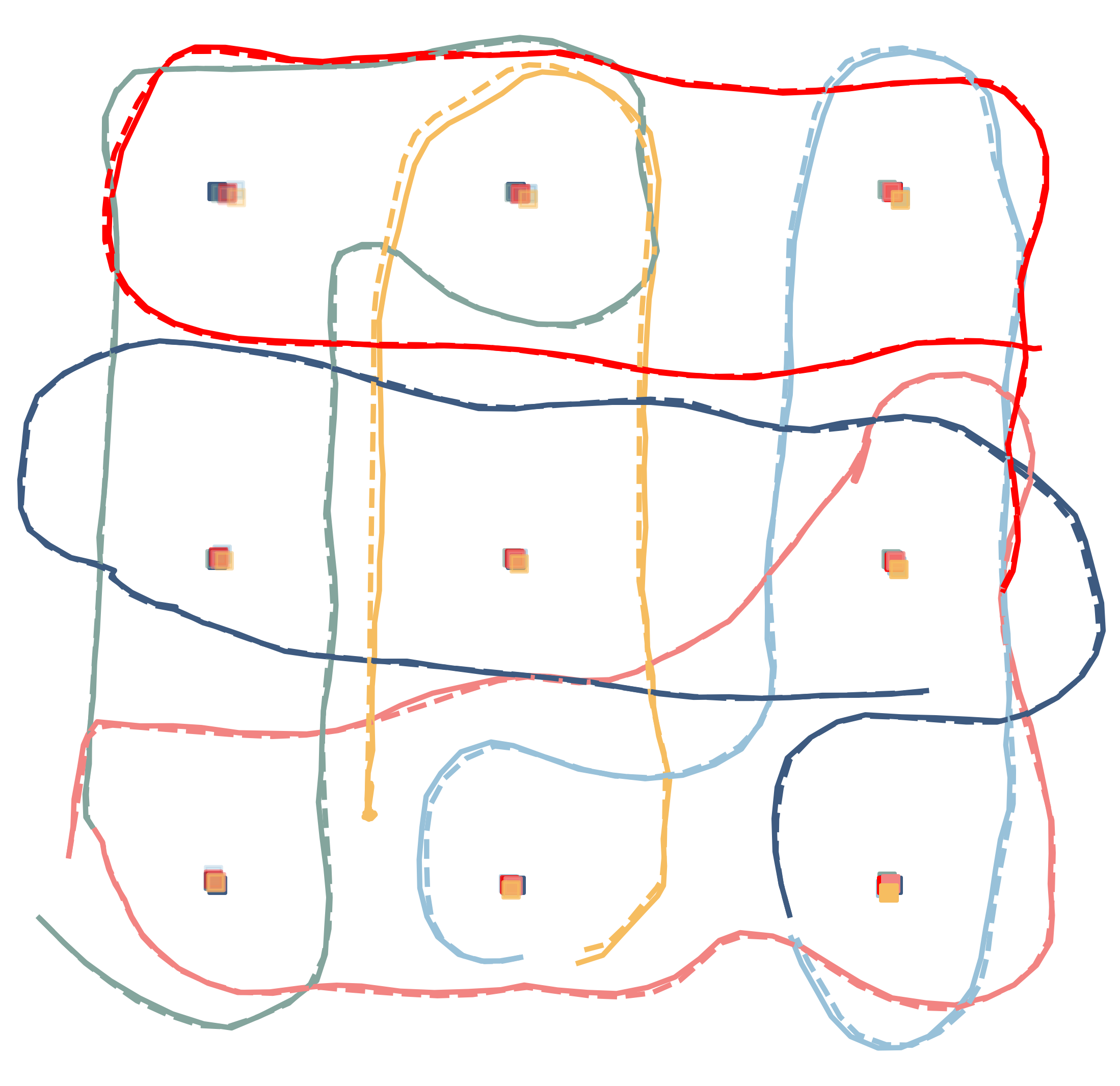}}
    \caption{The real-life experiment system comprises six Dingo robots and nine objects. The colored solid lines and square dots in (b) and (c) refer to the estimated trajectories and object positions of robots, respectively. The colored dashed lines and dark blue square dots refer to the true values of the robot trajectories and the object positions, respectively.}
\end{figure}

Recently, object-based SLAM has attracted increasing attention because persistent objects provide compact and semantically meaningful landmarks for long-term localization and mapping. 
In object-based multi-robot SLAM, the back-end must jointly estimate robot trajectories and object poses observed by multiple robots, leading to a tightly coupled pose-graph optimization (PGO) problem \cite{choudhary2017distributed, Tian2022DistributedRO}. 
A key feature of this problem is that its variables exhibit two different coupling structures: object poses are long-lived public variables shared across robots, whereas robot trajectories are high-dimensional private variables coupled only through sparse inter-robot loop closures. 
Existing decentralized PGO methods usually adopt either vertex-based splitting or edge-based splitting for the entire problem \cite{ziegler2021distributed,nguyen2022relative,mcgann_mesa_2024,choudhary2017distributed,tian2021distributed,Tian2022DistributedRO}. 
Such a uniform treatment does not fully exploit the variable structure of object-based PGO and may lead to unnecessary communication, especially when the communication graph is sparse, bandwidth-limited, or mismatched with the physical measurement graph.

Another challenge lies in the efficiency of decentralized optimization. 
First-order Riemannian methods preserve the geometry of $SE(3)$ but may require many communication rounds in ill-conditioned PGO problems \cite{tian2021distributed,9143442}. 
Approximate Newton methods can exploit local curvature information, but directly exchanging Hessians, gradients, or full pose graphs is impractical for decentralized multi-robot systems. 
Therefore, an effective object-based multi-robot PGO solver should jointly address three requirements: structure-aware communication reduction, geometry-preserving second-order optimization, and robustness under realistic communication constraints.

Motivated by these observations, this paper proposes DRAN, a decentralized Riemannian approximate Newton framework for object-based multi-robot PGO. 
The central insight is that object-based PGO should not be decoupled by a uniform splitting rule: shared objects and robot trajectories induce fundamentally different communication requirements. 
Objects are persistent public variables that require cross-robot consistency, whereas robot trajectories are high-dimensional private variables that only need to expose sparse separators at inter-robot loop closures. 
DRAN exploits this structure to build a communication-lean decentralized formulation and solves it using local Riemannian approximate Newton updates, where second-order information is used as a curvature preconditioner rather than being exchanged across the network. 
The main contributions are summarized as follows.

\begin{enumerate}[1)]
    \item \textbf{Object-trajectory-aware decentralized formulation:} 
    We characterize the different coupling structures of shared objects and private robot trajectories in object-based PGO, and derive a communication-efficient decoupling scheme that applies object-level consensus only to public variables and separator exchange only to trajectory boundary variables. 
    This formulation avoids unnecessary map or trajectory transmission under mismatched physical and communication topologies.

    \item \textbf{Decentralized Riemannian approximate Newton solver:} 
    We develop a fully decentralized solver on $SE(3)$ that constructs local LM/Gauss--Newton curvature models and computes curvature-preconditioned manifold updates without exchanging full Hessians, gradients, or pose graphs. 
    We prove the well-posedness of the reduced second-order update, establish convergence to Riemannian first-order stationary points, and provide a local condition-number analysis explaining the benefit over first-order Riemannian descent.

    \item \textbf{Communication-oriented evaluation under unreliable networks:} 
    We validate the method on public PGO benchmarks, large-scale object-based simulations, and a real multi-robot system. 
    Beyond accuracy and runtime, we report total transmitted data and evaluate performance under varying communication densities and probabilistic network interruptions, showing near-centralized accuracy with reduced communication load.
\end{enumerate}

\section{Related Works}\label{sec:Related Works}

\subsection{Decentralized PGO in Multi-Robot Systems}

Centralized PGO has matured with highly efficient nonlinear least-squares solvers \cite{g2o,gtsam} and certifiably correct relaxation-based methods (e.g., SE-Sync and Cartan-Sync), offering strong accuracy and scalability when global data aggregation is feasible \cite{Rosen19IJRR,cartan-sync}. 
However, in large-scale multi-robot applications, the central processor may become the bottleneck due to limited communication resources. To this end, Choudhary \emph{et al.} \cite{choudhary2017distributed} propose a two-stage distributed algorithm based on Jacobi over-relaxation (JOR) and successive over-relaxation (SOR) to seek approximate solutions for PGO. However, one step of the Gauss-Newton method in most cases cannot lead to sufficient convergence for distributed PGO. In addition, no line search is performed in \cite{choudhary2017distributed} due to the communication limitation, and thus, the behavior of the single Gauss-Newton step is totally unpredictable and might result in bad solutions. Tian~\emph{et~al.} \cite{tian2021distributed} develop a distributed certifiably correct PGO method that is guaranteed to converge to a globally optimal solution under moderate measurement noises. However, the communication networks of such localized distributed methods \cite{tian2021distributed,choudhary2017distributed, MM-PGO} need to align with the physical topology associated with the distributed PGO problem and might need to keep part of the robots in an idle state for estimate updates. In contrast, ADMM-based methods have been used not only for distributed computing\cite{Choudhary2015} but also to relieve the dependence on the communication topology\cite{ziegler2021distributed, nguyen2022relative} while optimizing in Euclidean space to approximate the solution of the PGO problem.

\vspace{-0.3cm}
\subsection{Decentralized (Object-based) PGO on Manifolds}
To account for constraints imposed on the rotation matrix in pose estimation, Riemannian optimization methods offer greater efficiency and precision due to their ability to leverage nonlinear properties and complex metric structures \cite{knuth2013collaborative, tron2012riemannian, tron2014distributed, 2020arXiv201000156C, mcgann_mesa_2024, Tian2022DistributedRO}. For instance, Knuth~\emph{et~al.}\cite{knuth2013collaborative} optimized local Lie group variables for each agent using a gradient descent method for PGO. To ensure consistency in public variables across multiple agents, consensus-based Riemannian gradient descent methods were introduced to estimate node poses in a camera network \cite{tron2012riemannian, tron2014distributed}, though these methods were found to be sensitive to noises. Likewise, Cristofalo~\emph{et~al.} \cite{2020arXiv201000156C} employed gradient descent on manifold spaces to maintain consistency, but this approach requires a high number of iterations due to the inherent limitations of gradient descent. McGann~\emph{et~al.} \cite{mcgann_mesa_2024} integrated Riemannian optimization into an ADMM-based framework, yet the performance was limited by the chosen optimizer. 
However, in most cases, these methods only estimate robots' poses and not the global map. 
Joint pose-and-map estimation is closer to distributed Structure-from-Motion (SfM) and Bundle Ajustment(BA), where distributed architectures (e.g., splitting/ADMM-style) have been studied for scalability \cite{wu2011multicore,eriksson2016consensus,zhang2017distributed}. 
Most relevant to our setting, LARPG leverages second-order information and separates public/private variables to reduce communication \cite{Tian2022DistributedRO}, but it still depends on a parameter-server architecture, motivating fully decentralized, communication-flexible, and second-order manifold optimization for collaborative object-based PGO.



\textbf{Notations:} We let $(\mathcal{M}, \langle \cdot , \cdot \rangle)$ denote a connected Riemannian manifold with Riemannian metric $\langle \cdot , \cdot \rangle$ and induced norm $\lVert \cdot \rVert$.
We use $\mathbb{R}^n$ to denote the $n$-dimensional Euclidean space. 
The Special Orthogonal group is defined as $SO(d)=\{R\in \mathbb{R}^{d\times d}:R^\top R=I,\det(R)=1\}$, and the Special Euclidean group is denoted as $SE(d)=\{(R,t):R\in SO(d), t\in \mathbb{R}^d\}$.
The tangent space at a point $x \in \mathcal{M}$ is denoted by $T_x\mathcal{M}$.
For a smooth scalar function $f : \mathcal{M} \rightarrow \mathbb{R}$, the Riemannian gradient $\text{grad} f(x) \in T_x\mathcal{M}$ represents the direction of steepest ascent.
We define the retraction $\text{Retr}_x: T_x\mathcal{M} \rightarrow \mathcal{M}$ as a smooth mapping that generalizes the exponential map, preserving first-order geometry.
Conversely, the inverse retraction (or logarithmic map) is denoted by $\text{Log}_x(\cdot) : \mathcal{M} \rightarrow T_x\mathcal{M}$, which maps a point on the manifold back to the tangent space.
Finally, for a positive definite information matrix $\Omega$, the weighted norm is defined as $\lVert v \rVert_{\Omega}^2 = \langle v, \Omega v \rangle$.
More details on Riemannian optimization can be found in \cite{absil2009optimization}.

\input{problem_formulation_v4}

\input{method_tie}

\input{convergence_analysis_3}

\input{experiments_tie}
\section{Conclusion}\label{sec:Conclusion}

This paper presented DRAN, a fully decentralized Riemannian approximate Newton framework for object-based multi-robot PGO. 
The key idea is to exploit the intrinsic variable structure of object-based PGO: shared objects are treated as public variables requiring Riemannian consensus, while robot trajectories remain private and interact only through sparse separator exchange. 
This object-trajectory-aware formulation reduces unnecessary communication under mismatched physical and communication topologies.
To solve the resulting decentralized problem, DRAN constructs local Newton approximation curvature models on the $SE(d)$ manifold and computes curvature-preconditioned updates without exchanging full Hessians, gradients, or pose graphs. 
We established the positive definiteness of the reduced second-order update, proved convergence to Riemannian first-order stationary points, and provided a local condition-number analysis showing how approximate second-order information improves the effective conditioning compared with first-order Riemannian descent.
Extensive experiments on public PGO benchmarks, large-scale object-based simulations, and a real multi-robot platform verified the accuracy, efficiency, and deployability of the proposed method. 
In particular, the communication-centered evaluations demonstrated that DRAN achieves near-centralized estimation performance with lower transmitted data and maintains robustness under sparse communication graphs and probabilistic network interruptions.

%% file: problem_formulation_v4.tex
\section{Problem Formulation}\label{sec:problem formulation}

We consider a collaborative object-based SLAM problem involving $N$ robots communicating over a connected graph $\mathcal{G}^c=(\mathcal{V}^c,\mathcal{E}^c)$. 
The system state consists of private robot trajectories $\bm{x} \triangleq \{x_i\}_{i=1}^N$ and shared environmental object poses $\bm{y} \triangleq \{y_l\}_{l=1}^M$, where $x_i \in SE(d)^{n_i}$ and $y_l \in SE(d)$. 
The global objective is to estimate these states by minimizing the Riemannian nonlinear least squares cost:
\begin{equation} \label{eq:global_cost}
    \min_{\bm{x}, \bm{y}} \sum_{k=1}^{K} \frac{1}{2} \lVert \varphi(T_{p_k}, T_{q_k}, \tilde{T}_{k}) \rVert_{\Omega_{k}}^2,  \ s.t.\ x \in \mathcal{X},\ y\in \mathcal{Y},
\end{equation}
where $T_{p_k}, T_{q_k} \in \{ \bm{x}, \bm{y} \}$ denote the estimated poses connected by the $k$-th measurement $\tilde{T}_k$ (e.g., odometry, object observation, or inter-robot loop) with precision matrix $\Omega_k$. 
The residual function $\varphi(\cdot)$ quantifies the discrepancy between the estimated and measured relative poses. 
Depending on the requirement for accuracy or computational efficiency, we define $\varphi(\cdot)$ using either the Riemannian geodesic distance or the chordal distance:
\begin{equation} \label{eq:residual_fun}
    \varphi(T_{p_k}, T_{q_k}, \tilde{T}_{k}) \triangleq 
    \begin{cases}
    \text{Log}\left( \tilde{T}_{k}^{-1} T_{p_k}^{-1} T_{q_k} \right), & \text{(Geodesic)} \\
    T_{p_k} \tilde{T}_{k} - T_{q_k}, & \text{(Chordal)}
    \end{cases}
\end{equation}
where $\text{Log}(\cdot)$ maps the error from the manifold $SE(d)$ to the tangent space $\mathfrak{se}(d)$ for standard Riemannian optimization, while the chordal formulation operates in the embedding space for robust initialization.
\subsection{Fine-Grained Distributed Decoupling}

\begin{figure}
    \centering
    \includegraphics[width=\columnwidth]{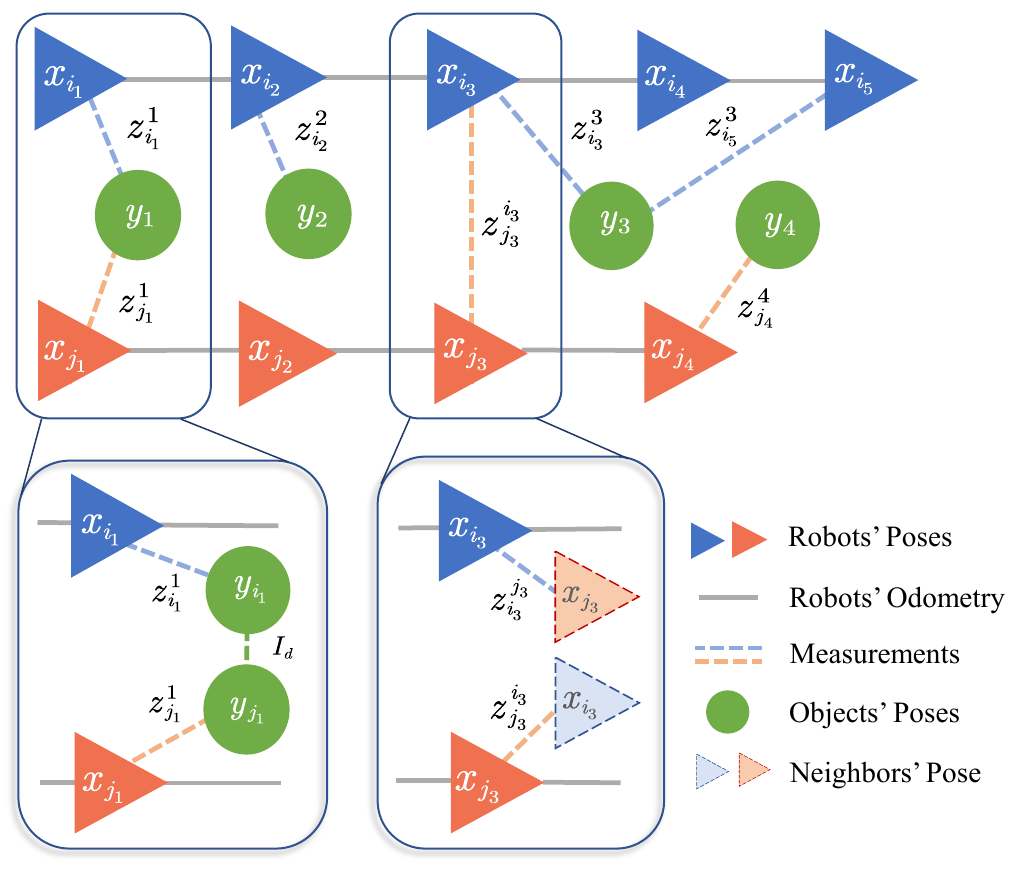} 
    \caption{A pose graph illustration of the multi-robot SLAM problem. Robot $i$ (blue triangle) and robot $j$ (red triangle) collaboratively estimate the poses of shared objects $\{y_1,\dots,y_4\}$ (green circles) and their respective trajectories $x_{i}$ and $x_{j}$. The graph includes intra-robot odometry (solid gray lines) and relative inter-robot/object measurements (colored dashed lines).}
    \label{fig:pgo_schematic}
\end{figure}

To enable fully decentralized optimization under bandwidth constraints, we propose a \textit{hybrid decoupling strategy} that treats map (objects) and trajectory (inter-robot) constraints differently (see Fig.~\ref{fig:pgo_schematic}). This formulation decomposes the global cost \eqref{eq:global_cost} into local sub-problems without sacrificing geometric consistency.

\subsubsection{Object Consistency (Vertex-Based Decoupling)}
Since objects $\bm{y}$ are static landmarks observed by multiple agents (e.g., green circles in Fig.~\ref{fig:pgo_schematic}), we treat them as public variables requiring global consensus. 
We employ a \textit{vertex-based} approach where each robot $i$ maintains a local copy $y_i \subseteq \bm{y}$ of the objects it observes. 
The consistency of the global map is enforced by constraining local copies of connected robots to be identical.
The local object-based cost $f^{\text{obj}}_i(x_i, y_i)$ aggregates measurements between the robot's trajectory $x_i$ and its local object copies $y_i$.

\subsubsection{Trajectory Connectivity (Edge-Based Decoupling)}
Inter-robot loop closures create direct coupling between the trajectories of different robots. Sharing full trajectories is communication-prohibitive. 
Instead, we adopt an \textit{edge-based} strategy using the concept of \textit{separators}.
Let $\mathcal{S}_i \subset x_i$ denote the subset of poses involved in inter-robot measurements with neighbors $j \in \mathcal{N}_i$. 
Robot $i$ only exchanges these separator estimates. During each local optimization iteration, the neighbors' separators $x_{\mathcal{S}_j}$ are treated as fixed priors (anchors). This yields the inter-robot cost term $f^{\text{inter}}_i(x_i | x_{\mathcal{N}_i})$, which optimizes local consistency relative to neighbors without requiring a full variable copy.

\subsection{Distributed Optimization Objective}
By combining the intra-robot odometry cost $f^{\text{intra}}_i(x_i)$ with the decoupled terms derived above, we formulate the distributed PGO as a constrained optimization problem on the manifold product space:

\begin{problem}[Distributed Pose Graph optimization]
\begin{equation}
\begin{aligned}
    &\min_{\bm{x},\bm{y}} f(\bm{x},\bm{y})  \triangleq \sum_{i=1}^N \underbrace{\left(f^{\text{intra}}_i(x_i)  +f^{\text{inter}}_i(x_i)  +f^{\text{obj}}_i(x_i,y_i)\right)}_{\triangleq f_i(x_i,y_i)}\\ 
    &s.t.\ x_i \in \mathcal{X}_i,\ y_i\in \mathcal{Y}_i,\  \varphi(y_i,y_j)=0, \forall i,j \in [N].
    \label{eq:Collaborative Geometric Estimation}
\end{aligned}
\end{equation}
where $f^{\text{inter}}_i(x_i)\triangleq f(x_i|x_{\mathcal{N}_i})$, $\boldsymbol{x}\in \mathcal{X}$ and $\boldsymbol{y}\in\mathcal{Y} $ represent the collective vector of $x_i,\ y_i$ for $i \in [N]$, respectively. 
\label{problem_formulation2}
\end{problem} 

Here, $f_i(x_i, y_i)$ represents the total local cost for robot $i$. The constraint $\varphi(y_i, y_j) = 0$ ensures that the distributed object maps converge to a consistent global reference frame. 
To solve this consensus problem over time-varying networks, we utilize a Metropolis-Hastings weight matrix $W$, which robustly balances information fusion based on node connectivity \cite{weight_matrix}. 
This formulation explicitly decouples the high-dimensional private trajectories from the shared public map, minimizing communication overhead while maintaining the mathematical structure required for precise Riemannian optimization.

%% file: method_tie.tex
\section{Methodology}\label{sec:Method}

In this section, we introduce our decentralized PGO method, which leverages consensus theory and approximate Riemannian Hessian information to solve Problem \ref{problem_formulation2}.


\vspace{-
0.3cm}
\subsection{Algorithm Design}

Similar with the edge-based ADMM framework~\cite{EADMM}, we define an augmented Lagrangian for the global cost function $f(\bm{x}, \bm{y})$ as follows:
\begin{equation}\label{eq:lagrangian_global}
\begin{aligned}
    L\left(\bm{x}, \bm{y},\bm{\lambda}\right) &=f(\bm{x},\bm{y}) +\sum_{i=1}^{n} \sum_{j\in \mathcal{N}_i} \left< \lambda _{ij},w_{ij}\varphi \left( y_i,y_j \right) \right>\\
    &\quad+\sum_{i=1}^{n} \sum_{j\in \mathcal{N}_i}\frac{\beta}{2}\lVert w_{ij}\varphi\left( y_i,y_j \right) \rVert ^2,
\end{aligned}
\end{equation}
where $\lambda _{i}=[\lambda_{ij}]_{j\in\mathcal{N}_i} $ denotes the dual variable associated with robot $i$, $\bm{\lambda}:=[\lambda _{i}]_{i\in [N]}$ is the stack vector of all dual variables and $\beta$ is the penalty coefficient. The weights $w_{ij}$ are introduced in the augmented Lagrangian function to modulate the influence of the consistency constraint between neighboring robots $i$ and $j$. The process of solving the augmented Lagrangian in Eq.~\eqref{eq:lagrangian_global} can be expressed as:
\begin{subequations}
\begin{align}
    y^{k+1}_i,x^{k+1}_i&= \mathop{\arg\min}_{y_i ,x_i} L\left(\bm{x}^k, \bm{y}^k,\bm{\lambda}^k \right) \label{eq:primal_update},\\
    \lambda _{ij}^{k+1}&=\lambda _{ij}^{k}+\eta\beta w_{ij}\varphi\left( y_i^{k+1},y_j^{k+1} \right), \label{eq:lambda_update}
\end{align}
\end{subequations}
where $k$ is the iteration index in the optimization process, and $\rho>0$ denotes the stepsize of updating dual variables.

By rewriting the dual and penalty terms as a `Biased Prior', we can represent the augmented Lagrangian in Eq.~\eqref{eq:lagrangian_global} using a factor-graph~\cite{Choudhary2015}. Specifically, knowing that, when $b$ is constant, $\arg\min_a \left<b,a\right>+(\beta/2)\lVert a\rVert ^2=\arg\min_a (\beta/2)\lVert a+b/\beta\rVert ^2 $, we define the local Lagrangian function by
\begin{equation}\label{eq:Biased Prior}
\begin{aligned}
    \hat{L}_i\left( x_i^k,y_i^k,\lambda _i^k \right) &=f_i\left( x_i^k,y_i^k \right)\\
    &\quad+\sum_{j\in \mathcal{N}_i}\frac{\beta}{2}\lVert w_{ij}\varphi \left( y_i^k,y_j^k \right)+\frac{\lambda _{ij}^k}{\beta} \rVert ^2 .
\end{aligned}
\end{equation}

However, the minimization steps in Eq.~\eqref{eq:primal_update} are still non-trivial, as we can not derive a closed form for $x_i^{k+1}$ and $ y_i^{k+1}$, even though $\hat{L}_i$ is quadratic. Thus, we replace the functions $\hat{L}_i,i\in[N]$ with their approximate quadratic functions $m_i$ evaluated at $x_i^{k}$ and $y_i^{k}$ as follows.
\begin{equation}\label{eq:quadratic_model}
\begin{aligned}
    m_i\left( u_i,v_i \right) =&\hat{L}_i\left( x_i^k,y_i^k,\lambda _i^k \right) +\left< \underbrace{\left[ \begin{array}{c}
    	g_{ix}\\
    	g_{iy}\\
    \end{array} \right]}_{g_i} ,\left[\begin{array}{c}
    	u_i\\
    	v_i\\
    \end{array} \right] \right>\\
    &+\frac{1}{2}\left< \left[ \begin{array}{c}
    	u_i\\
    	v_i\\
    \end{array} \right] ,\underbrace{\left[ \begin{matrix}
    	A_i&		C_i\\
    	C_{i}^{\top}&		B_i\\
    \end{matrix} \right]}_{M_i} \left[ \begin{array}{c}
    	u_i\\
    	v_i\\
    \end{array} \right] \right>,     
\end{aligned}
\end{equation}
where $(u_i, v_i) \in T_{x_i} \mathcal{X}_i \times T_{y_i} \mathcal{Y}_i$ are the tangent vectors, $g_i \triangleq grad\ f_i(x_i, y_i)$ is the local Riemannian gradient. The user-specified linear map $M_i \succ 0$ serves as an approximation of the local Riemannian Hessian and is assumed to be symmetric and positive definite. In particular, the approximation of Riemannian Hessian can be obtained via the Riemannian Levenberg--Marquardt (LM) method~\cite[Chapter~8]{absil2009optimization}, i.e., $M_i = {J_i} ^\top J_i + \mu I$, where $J_i$ is the Jacobian of agent $i$'s measurement residuals and $\mu > 0$ is a regularization parameter that ensures $M_i$ to be positive definite.

Then, we aim to compute an update for all variables by approximately minimizing $m_i$. However, directly applying a Newton-like method to all variables is computationally prohibitive because of the large dimensionality of the Riemannian Hessian matrix $M_i$. As a result, we use the Schur complement method \cite{g2o,ddfsam2,Tian2022DistributedRO} to decouple the public and private variables so as to calculate the inverse of the Hessian matrix efficiently. While the Schur complement method can also be computationally expensive, our experimental results demonstrate that it significantly reduces the average computation time while maintaining comparable accuracy.

Specifically, we minimize the approximate function $m_i$ in Eq.~\eqref{eq:quadratic_model} with $v_i$ being fixed to eliminate private vector $u_i$. Setting the gradient of $m_i(u_i,v_i)$ with respect to $u$ to zero yields $\nabla_{u}m_i(u_i,v_i)=g_{ix}+A_iu+Cv_i=0$. Letting $u_i^* \triangleq \mathop{\arg\min}_{u_i} m_i(u_i,v_i)$ denote the optimal private vector conditioned on the public vector, we have
\begin{equation}\label{eq:private_delta}
    u_i^*\left( v_i \right) =-A_{i}^{-1}\left( C_iv_i+g_{ix} \right) .
\end{equation}
Next, we define the reduced second-order approximation as $h_i(v_i) \triangleq m_i(u_i^*(v_i), v_i)$, which depends solely on the public vector $v_i$, i.e.,
\begin{equation}\label{eq:hivi}
\begin{split}
        h_i(v_i)=&\hat{L}_i\left( x_i,y_i \right) -\frac{1}{2}\left< g_{ix},A_{i}^{-1}g_{ix} \right>\\
        &+\left< \hat{g}_i,v_i \right> +\frac{1}{2}\left< v_i,\hat{H}_iv_i \right>, 
\end{split}
\end{equation}
where the reduced gradient $\hat{g}_i$ and the reduced Hessian $\hat{H}_i$ are defined in a similar way, respectively, as follows:
\begin{equation}
    \hat{g}_i=g_{iy}-C_{i}^{\top}A_{i}^{-1}g_{ix},
\end{equation}
\begin{equation}
    \hat{H}_i=B_i-C_{i}^{\top}A_{i}^{-1}C_i.
\end{equation}
Further, we can minimize $h_i(v_i)$ by computing $v_i= \hat{H}_{i}^{-1}\hat{g}_i$, and in turn minimize $m_i$ by updating private variables $x$ with $u_i^*$ (c.f., Line 10-13 in Alg.~\ref{alg:NT1}). 

The behavior of consensus strategies on Riemannian manifolds differs significantly from that in Euclidean spaces due to the effect of curvature. To establish a formal procedure for achieving consensus, 
We measure consistency error based on the Fréchet Mean as defined in \cite{tron2012riemannian}, and the Riemannian gradient $g_{iy}$ of $\hat{L}_i$ on $y$ is presented as follows,
\begin{equation}
\begin{aligned}
    g_{iy}&=grad_yf_i\left( x_i,y_i \right) \\
    &+\beta \sum_{j\in \mathcal{N}_i}{w_{ij}\left( w_{ij}\varphi ^2\left( y_i,y_j \right) +\frac{\lambda _{ij}}{\beta} \right) grad_{y_i}\varphi ^2\left( y_i,y_j \right)},
\end{aligned}
\end{equation}
where $grad_yf_i\left( x_i,y_i \right)$ represents the gradient of the local objective function $f_i(x_i,y_i)$, and the second term captures the gradient of the consistency error. Note that, when geodesic distance is used for $\varphi(\cdot)$, the term $grad_{y_i}\varphi ^2\left( y_i,y_j \right)$ can be explicitly calculated as $2Log_{y_i}(y_j)$. The value of $\frac{\lambda _{ij}}{\beta}$ acts as a fixed bias term, which provides a trade-off between consensus of public variables and local updates.

\begin{algorithm}[t]
    \renewcommand{\algorithmicrequire}{\textbf{Input:}}
    \renewcommand{\algorithmicensure}{\textbf{Output:}}
	\caption{\textsc{Decentralized Riemannian Approximate Newton Method (DRAN)} }
	\label{alg:NT1}
	\begin{algorithmic}[1]
		\small 
        \Require Communication networks $\mathcal{G}^c=(\mathcal{V}^c,\mathcal{E}^c)$, initial state $x^0, y^0,\lambda^0=\boldsymbol{0}$, $K$, stopping criteria parameter $\epsilon_u$, $\epsilon_v$, $\epsilon_{\varphi}$
	    \Ensure Globally consistent object-based maps and optimized trajectories
		\For{iteration $k = 0, 1, \hdots,K $}
			\For{each agent $i$ \textbf{in parallel}}
                    \State Exchange estimated separators, object with neighbors
                        \For{object received from neighbors $j\in \mathcal{N}_i$}
                            \If{object is previously unrecognized by robot $i$}
                                \State Initialize the object pose $y_{i}^k$ with $y_{j}^k$
                            \EndIf
                        \EndFor
                    \State Update dual variables $\lambda_i^{k+1}$ as \eqref{eq:lambda_update}
                    \State Computes reduced gradient $\hat{g}_i$ and reduced Hessian $\hat{H}_i$
                    \State Computes the $v^{k+1}_i$ as 
                    $v^{k+1}_i=-\hat{H}_{i}^{-1}\hat{g}_i$
                    \State Computes the $u^{k+1}_i$ as \eqref{eq:private_delta}
                    \State \label{alg:update} Update public variables and private variables as 
                    \[y_i^{k+1} = Retr_{y_i^{k}}(-\alpha v_i^{k+1}),\ x_i^{k+1} = Retr_{x_i^{k}}(-\gamma u_i^{k+1})\]  
		      \EndFor
                \If{stopping criteria are satisfied}
                    \State \textbf{break}
                \EndIf
		\EndFor
	\end{algorithmic}
\end{algorithm}

The proposed Decentralized Riemannian Approximate Newton (DRAN) method (c.f., Alg.~\ref{alg:NT1}) achieves accelerated convergence by gradually aligning local second-order approximations with global Newton directions via Riemannian consensus. Specifically, each robot $i$ locally solves the equation $\hat{H}_{i}v_i= \hat{g}_i$. Summing it over all robots yields $\sum_{i=1}^{N}\hat{H}_iv_i=\sum_{i=1}^{N}\hat{g}_i$. Therefore, when $v_i$ becomes sufficiently consistent across robots, it approximates the global Newton direction $\Bar{v}=(\sum_{i=1}^{N}\hat{H}_i)^{-1}\sum_{i=1}^{N}\hat{g}_i$, which enables efficient and scalable distributed optimization. 

\begin{remark}
     Unlike \cite{Choudhary2015}, which relies on a generic Gauss-Newton solver, our method explicitly exploits the Riemannian geometry of the problem and employs a Newton-like update tailored to the structure of the PGO problem, leading to faster convergence. Additionally, in contrast to \cite{Tian2022DistributedRO}, the local Hessian matrix and local gradient do not need to be communicated to neighboring robots. Moreover, the consensus scheme used in Lin~\emph{et~al.}\cite{9442938} can not be employed to achieve consistency in Riemannian space, which complicates algorithm design and implementation. Thus, we construct a residual function directly defined on Riemannian space (c.f., Eq.~\eqref{eq:residual_fun}) to ensure the consistency of public variables of all nodes. 
\end{remark}

\subsection{Implementation of DRAN}

To ensure robust convergence and adaptability in dynamic environments, we implement DRAN with a specific focus on high-quality initialization, flexible communication, and rigorous termination criteria.

\subsubsection{Two-Stage Distributed Initialization}
Given the non-convex nature of PGO on the $SE(d)$ manifold, the quality of initial estimates is critical for convergence. We adopt a \textit{two-stage distributed initialization} strategy extending the chordal relaxation approach \cite{Carlone2015Initialization}. 
First, we relax the rotation synchronization into a linear least-squares problem, solving it distributedly to obtain initial rotations. 
Second, fixing these rotations, we solve for translations via a distributed Gauss-Seidel method. 
This decoupled procedure efficiently provides a high-quality initial guess $(x^0, y^0)$ for the subsequent Riemannian optimization, significantly reducing the risk of getting stuck in local minima.

\subsubsection{Dynamic Map Consensus Mechanism}
In decentralized settings, partial observability is inevitable. To address this, we integrate a dynamic map consensus mechanism into each communication round. 
Robots exchange estimated separators and object poses with neighbors. When a robot encounters a previously unobserved object ID, it initializes its local estimate using the neighbor's data (see Line 3-8 in Alg.~\ref{alg:NT1}). 
Crucially, robots act as information relays, allowing map estimates to propagate across the network within $d-1$ rounds (where $d$ is the network diameter). 
This strategy ensures global map consistency even under sparse, topology-varying communication independent of the physical interaction graph. This framework is versatile; in object-free scenarios, the algorithm seamlessly degenerates to standard distributed PGO by omitting the object-related cost terms $f^{\text{obj}}$, optimizing only private trajectories via separator exchange.

\subsubsection{Termination Criteria}
The optimization terminates when either the \textit{optimality} or \textit{feasibility} condition is satsfied, i.e.,
\begin{itemize}
    \item \textbf{Optimality:} The update steps on the manifold become negligible, i.e., $\|u_i\| \le \epsilon_u$ and $\|v_i\| \le \epsilon_v$, indicating convergence to a local minimum.
    \item \textbf{Feasibility:} The consensus error of shared variables drops below a tolerance, i.e., $\varphi^2(y_i, y_j) \le \epsilon_{\varphi}$, ensuring geometric consistency.
\end{itemize}
In our experiments, we set thresholds $\epsilon_u = 10^{-2}$, $\epsilon_v = 10^{-1}$, $\epsilon_{\varphi} = 10^{-1}$, and penalty parameters $\beta=1, \eta=0.1$. A safety cap of $K_{\max}=500$ iterations is enforced to bound computation time.

%% file: convergence_analysis_3.tex
\section{Convergence Analysis}
\label{sec:convergence}

In this section, we analyze the proposed decentralized Riemannian approximate Newton method. 
The purpose of the analysis is threefold. 
First, we show that the Levenberg--Marquardt-type approximate Hessian and the Schur complement reduction are well-defined and preserve positive definiteness on the tangent spaces. 
Second, we prove that the reduced approximate Newton direction provides a curvature-preconditioned descent direction and leads to a non-asymptotic first-order stationarity bound with an explicit curvature-dependent constant. 
Third, for time-varying communication graphs, we establish a tracking bound showing how the distributed reduced Newton direction approaches the centralized reduced Newton direction as a function of the graph mixing rate and the temporal variation of local curvature information.

\subsection{Notation and Preliminaries}

For the convergence analysis, we collect only the notation needed later. 
Let
    $z_i \triangleq (x_i,y_i)\in \mathcal M_i\triangleq \mathcal X_i\times\mathcal Y_i$
denote the local state of robot $i$, and let $\hat L_i(z_i)$ be the local augmented objective defined in the previous section, with the dual variable omitted for notational simplicity. 
At iteration $k$, define the tangent step and the Riemannian gradient as
\[
    \xi_i^k
    \triangleq
    \begin{bmatrix}
        u_i^k\\
        v_i^k
    \end{bmatrix}
    \in T_{z_i^k}\mathcal M_i,
    \qquad
    g_i^k
    \triangleq
    \operatorname{grad}\hat L_i(z_i^k)
    =
    \begin{bmatrix}
        g_{ix}^k\\
        g_{iy}^k
    \end{bmatrix}.
\]

As introduced in the method section, DRAN builds the local quadratic model
\begin{equation}
\label{eq:conv_local_model}
    m_i^k(\xi_i)
    =
    \hat L_i(z_i^k)
    +
    \langle g_i^k,\xi_i\rangle
    +
    \frac{1}{2}
    \langle \xi_i,M_i^k\xi_i\rangle ,
\end{equation}
where
\begin{equation}
\label{eq:conv_block_hessian}
    M_i^k=
    \begin{bmatrix}
        A_i^k & C_i^k\\
        (C_i^k)^\top & B_i^k
    \end{bmatrix}
\end{equation}
is the local LM-type Hessian approximation on $T_{z_i^k}\mathcal M_i$. 
In particular,
\begin{equation}
\label{eq:conv_lm_hessian}
    M_i^k=(J_i^k)^\top J_i^k+\mu_i^k I,
    \qquad
    \mu_i^k>0 .
\end{equation}

The reduced approximate Newton direction and the recovered private direction are
\begin{equation}
\label{eq:conv_local_newton}
    v_{i,N}^k
    =
    -(\hat H_i^k)^{-1}\hat g_i^k,
    \qquad
    u_{i,N}^k
    =
    -(A_i^k)^{-1}(C_i^k v_{i,N}^k+g_{ix}^k).
\end{equation}
Equivalently, the full tangent direction satisfies
\begin{equation}
\label{eq:conv_full_newton_equiv}
    \xi_{i,N}^k
    \triangleq
    \begin{bmatrix}
        u_{i,N}^k\\
        v_{i,N}^k
    \end{bmatrix}
    =
    -(M_i^k)^{-1}g_i^k .
\end{equation}
The corresponding Riemannian update is written as
\begin{equation}
\label{eq:conv_retraction_update}
    z_i^{k+1}
    =
    \operatorname{Retr}_{z_i^k}(\alpha \xi_{i,N}^k),
    \qquad
    0<\alpha\le 1 .
\end{equation}

Throughout this section, $\xi_{i,N}^k$ denotes the descent direction. 
Thus, if one instead defines $d_i^k=(M_i^k)^{-1}g_i^k$, the update should equivalently be written as 
$z_i^{k+1}=\operatorname{Retr}_{z_i^k}(-\alpha d_i^k)$.

\subsection{Positive Definiteness of the Reduced Hessian}

We first show that the approximate Hessian used in DRAN remains positive definite even when $(J_i^k)^\top J_i^k$ is singular.

\begin{assumption}[Bounded residual Jacobian and LM damping]
\label{ass:conv_lm}
For every robot $i$ and iteration $k$, the local residual Jacobian satisfies $\|J_i^k\|\le \bar J$ ,
and the damping parameter satisfies
\[
    0<\underline \mu \le \mu_i^k\le \bar \mu <+\infty .
\]
\end{assumption}

\begin{lemma}[Positive definiteness of the LM approximation]
\label{lem:conv_lm_spd}
Under Assumption~\ref{ass:conv_lm}, the approximate Hessian
\[
    M_i^k=(J_i^k)^\top J_i^k+\mu_i^k I
\]
satisfies
\begin{equation}
\label{eq:conv_M_bounds}
    c_m I
    \preceq
    M_i^k
    \preceq
    c_M I,
\end{equation}
where $c_m=\underline\mu$, $c_M=\bar J^2+\bar\mu .$
\end{lemma}

\noindent\emph{Proof.} The proof is provided in Appendix~\ref{app:proof:lem:conv-lm-spd}.

\begin{lemma}[Positive definiteness of the Schur complement]
\label{lem:conv_schur_spd}
Suppose $M_i^k$ satisfies \eqref{eq:conv_M_bounds} and is partitioned as in \eqref{eq:conv_block_hessian}. 
Then $A_i^k\succ0$, and the reduced Hessian
\[
    \hat H_i^k
    =
    B_i^k
    -
    (C_i^k)^\top(A_i^k)^{-1}C_i^k
\]
is symmetric positive definite. 
Moreover,
\begin{equation}
\label{eq:conv_reduced_bounds}
    c_m I
    \preceq
    \hat H_i^k
    \preceq
    c_M I .
\end{equation}
\end{lemma}

\noindent\emph{Proof.} The proof is provided in Appendix~\ref{app:proof:lem:conv-schur-spd}.

\subsection{Global Descent and First-Order Stationarity}

We now prove that the damped Riemannian approximate Newton step is a descent step. 
The result is global in the sense of nonconvex Riemannian optimization: it guarantees convergence to a first-order stationary point, but it does not claim quadratic convergence.

\begin{assumption}[Pullback smoothness]
\label{ass:conv_pullback_smooth}
For every robot $i$ and iteration $k$, the pullback function
\[
    \tilde L_i^k(\xi)
    \triangleq
    \hat L_i\bigl(\operatorname{Retr}_{z_i^k}(\xi)\bigr),
    \qquad
    \xi\in T_{z_i^k}\mathcal M_i,
\]
satisfies
\begin{equation}
\label{eq:conv_pullback_smooth}
    \tilde L_i^k(\xi)
    \le
    \hat L_i(z_i^k)
    +
    \langle g_i^k,\xi\rangle
    +
    \frac{L}{2}\|\xi\|^2 .
\end{equation}
\end{assumption}

\begin{assumption}[Lower bounded local augmented objective]
\label{ass:conv_lower_bound}
For every robot $i$, there exists $\hat L_i^\star>-\infty$ such that
    $\hat L_i(z_i^k)\ge \hat L_i^\star$
for all iterations $k$.
\end{assumption}

\begin{theorem}[Curvature-preconditioned descent]
\label{thm:conv_descent}
Suppose Assumptions~\ref{ass:conv_lm}--\ref{ass:conv_lower_bound} hold. 
Let $\xi_{i,N}^k=-(M_i^k)^{-1}g_i^k$ and let the update be
\[
    z_i^{k+1}
    =
    \operatorname{Retr}_{z_i^k}(\alpha \xi_{i,N}^k).
\]
If
\begin{equation}
\label{eq:conv_stepsize}
    0<\alpha<\frac{2c_m^2}{Lc_M},
\end{equation}
then
\begin{equation}
\label{eq:conv_descent_bound}
    \hat L_i(z_i^{k+1})
    \le
    \hat L_i(z_i^k)
    -
    \rho_{\rm AN}
    \|g_i^k\|^2,
\end{equation}
where
\begin{equation}
\label{eq:conv_rho_an}
    \rho_{\rm AN}
    \triangleq
    \frac{\alpha}{c_M}
    -
    \frac{L\alpha^2}{2c_m^2}
    >
    0 .
\end{equation}
\end{theorem}

\noindent\emph{Proof.} The proof is provided in Appendix~\ref{app:proof:thm:conv-descent}.

\begin{corollary}[Non-asymptotic first-order stationarity]
\label{cor:conv_stationarity}
Under the same conditions of Theorem~\ref{thm:conv_descent}, for every robot $i$ and every $K\ge1$,
\begin{equation}
\label{eq:conv_stationarity}
    \min_{0\le k\le K-1}
    \|\operatorname{grad}\hat L_i(z_i^k)\|^2
    \le
    \frac{
        \hat L_i(z_i^0)-\hat L_i^\star
    }{
        \rho_{\rm AN}K
    } .
\end{equation}
\end{corollary}

\noindent\emph{Proof.} The proof is provided in Appendix~\ref{app:proof:cor:conv-stationarity}.

\subsection{Effect of Second-Order Information}
\label{subsec:second_order_effect}

While Corollary~\ref{cor:conv_stationarity} establishes a standard worst-case stationarity rate of $O(1/K)$ due to the nonconvexity of the global problem, the second-order approximation significantly alters the local geometry. Specifically, in a regular local region, the LM/Gauss--Newton approximation in DRAN acts as a Riemannian preconditioner that improves the effective condition number of the local model.

Furthermore, since relative-pose measurements are invariant to global $SE(d)$ transformations, the inherent gauge freedom renders the unconstrained PGO Hessian singular. Therefore, our subsequent local conditioning analysis is strictly conducted on the gauge-fixed tangent space, achieved by anchoring a reference pose, adding an equivalent prior, or restricting the analysis to the observable subspace.

Let $H_i^k$ denote the Riemannian Hessian of the pullback objective at the origin of $T_{z_i^k}\mathcal M_i$, restricted to the gauge-fixed tangent space. 
For nonlinear least-squares PGO, this Hessian can be decomposed as
\begin{equation}
\label{eq:true_hessian_decomposition}
    H_i^k
    =
    (J_i^k)^\top J_i^k
    +
    R_i^k,
\end{equation}
where $J_i^k$ is the Jacobian of the local residuals and $R_i^k$ collects the second-order residual terms. 
The LM approximation used by DRAN is
\begin{equation}
\label{eq:lm_approx_again}
    M_i^k
    =
    (J_i^k)^\top J_i^k+\mu_i^k I,
    \qquad
    \mu_i^k\ge 0.
\end{equation}

\begin{assumption}[Local regularity on the gauge-fixed tangent space]
\label{ass:local_regular_gauge_fixed}
In a neighborhood of a nondegenerate local solution to the PGO Problem~\ref{problem_formulation2}, restricted to the gauge-fixed tangent space, the pullback Hessian satisfies
\begin{equation}
\label{eq:gauge_fixed_hessian_bounds}
    m_H I
    \preceq
    H_i^k
    \preceq
    L_H I,
    \qquad
    0<m_H\le L_H<+\infty .
\end{equation}
\end{assumption}

\begin{assumption}[Accuracy of the LM/Gauss--Newton approximation]
\label{ass:gn_lm_accuracy}
In the same local neighborhood, restricted to the gauge-fixed tangent space, the second-order residual term and the LM damping satisfy
\begin{equation}
\label{eq:residual_hessian_bound}
    \|R_i^k\|
    =
    \left\|
    H_i^k-(J_i^k)^\top J_i^k
    \right\|
    \le
    \varepsilon_H^k,
\end{equation}
and $0\le \mu_i^k\le \varepsilon_\mu^k$.

\end{assumption}

\begin{remark}
The two assumptions above are local and standard for nonlinear least-squares PGO\cite{Rosen19IJRR,dennis1996numerical,nocedal2006numerical}. 
After anchoring one pose, adding an equivalent prior, or restricting the analysis to the observable tangent space, a nondegenerate local solution yields a positive definite Hessian with bounded spectrum,
$m_H I\preceq H_i^k\preceq L_H I$. 
The lower bound removes unobservable or degenerate directions, while the upper bound follows from residual smoothness in a compact neighborhood\cite{Rosen19IJRR}. 
Moreover, for least-squares objectives,
$H_i^k=(J_i^k)^\top J_i^k+R_i^k$. 
The residual term $R_i^k$ is bounded within the neighborhood of the regular solution, and the LM damping is bounded by design. 
Hence $\|R_i^k\|\le\varepsilon_H^k$ and $0\le\mu_i^k\le\varepsilon_\mu^k$ quantify the local accuracy of the LM/Gauss--Newton approximation\cite{dennis1996numerical,nocedal2006numerical}.
\end{remark}

\begin{lemma}[Local spectral equivalence of the LM approximation]
\label{lem:lm_spectral_equivalence}
Suppose Assumptions~\ref{ass:local_regular_gauge_fixed} and \ref{ass:gn_lm_accuracy} hold. 
If, for some $0<\delta<1$,
\begin{equation}
\label{eq:spectral_equiv_condition}
    \varepsilon_H^k+\varepsilon_\mu^k
    \le
    \delta m_H,
\end{equation}
then
\begin{equation}
\label{eq:lm_spectral_equivalence}
    (1-\delta)H_i^k
    \preceq
    M_i^k
    \preceq
    (1+\delta)H_i^k .
\end{equation}
\end{lemma}

\noindent\emph{Proof.} The proof is provided in Appendix~\ref{app:proof:lem:lm-spectral-equivalence}.

\begin{theorem}[Conditioning effect of approximate Newton preconditioning]
\label{thm:conv_conditioning}
Suppose Assumptions~\ref{ass:local_regular_gauge_fixed}
and~\ref{ass:gn_lm_accuracy} hold.
If, for some $0<\delta<1$,
\[
    \varepsilon_H^k+\varepsilon_\mu^k \le \delta m_H,
\]
then $M_i^k$ is spectrally equivalent to $H_i^k$, and the preconditioned local Hessian satisfies
\begin{equation}
\label{eq:conv_preconditioned_spectrum}
    \frac{1}{1+\delta}I
    \preceq
    (M_i^k)^{-1/2}H_i^k(M_i^k)^{-1/2}
    \preceq
    \frac{1}{1-\delta}I .
\end{equation}
Consequently, the condition number of the preconditioned local model is bounded by
\begin{equation}
\label{eq:conv_preconditioned_condition}
    \kappa_{\rm AN}
    \le
    \frac{1+\delta}{1-\delta}.
\end{equation}
In contrast, an unpreconditioned Riemannian gradient step is governed by the condition number
\begin{equation}
\label{eq:conv_gradient_condition}
    \kappa_{\rm GD}
    =
    \frac{L_H}{m_H}.
\end{equation}
\end{theorem}

\noindent\emph{Proof.} The proof is provided in Appendix~\ref{app:proof:thm:conv-conditioning}.

\begin{remark}[Connection to the DRAN update]
\label{rem:connection_dran_conditioning}
The DRAN update is obtained by minimizing the local quadratic model
\[
    m_i^k(\xi_i)
    =
    \hat L_i(z_i^k)
    +
    \langle g_i^k,\xi_i\rangle
    +
    \frac{1}{2}\langle \xi_i,M_i^k\xi_i\rangle .
\]
Thus, the resulting full-space direction is
\[
    \xi_i^k
    =
    -(M_i^k)^{-1}g_i^k,
\]
which can be interpreted as a Riemannian preconditioned gradient direction with preconditioner $(M_i^k)^{-1}$. 
The matrix
\[
    (M_i^k)^{-1/2}H_i^k(M_i^k)^{-1/2}
\]
is the effective Hessian approximation after this preconditioning. 
If $M_i^k$ is spectrally close to $H_i^k$, this effective Hessian is close to the identity, meaning that the local curvature is approximately normalized. 
This is the theoretical mechanism by which approximate second-order information can reduce the number of optimization iterations.
\end{remark}

\begin{remark}[Interpretation of the Acceleration Effect]
Theorem~\ref{thm:conv_descent} establishes that the objective strictly decreases by at least $\rho_{\rm AN} \|g_i^k\|^2$. To intuitively understand the acceleration, consider the optimal step size $\alpha^* = \frac{c_m^2}{L c_M}$, which yields a descent factor of:
\begin{equation}
    \rho_{\rm AN}^* = \frac{1}{2L} \left( \frac{c_m}{c_M} \right)^2.
\end{equation}
When the iterates enter a regular local region, Lemma~\ref{lem:lm_spectral_equivalence} and Theorem~\ref{thm:conv_conditioning} show that the LM approximation acts as an preconditioner. By transforming the problem into the preconditioned local geometry, the curvature matrix has explicitly bounded eigenvalues from Equation~\eqref{eq:conv_preconditioned_spectrum}. Thus, the effective local bounds become $\tilde{c}_m = \frac{1}{1+\delta}$ and $\tilde{c}_M = \frac{1}{1-\delta}$. 

Consequently, the condition number of the preconditioned update is tightly bounded by $\kappa_{\rm AN} = \tilde{c}_M / \tilde{c}_m \le \frac{1+\delta}{1-\delta}$. Substituting this localized effective condition number into our descent factor gives:
\begin{equation}
    \rho_{\rm AN, local}^* \approx \frac{1}{2L} \left( \frac{\tilde{c}_m}{\tilde{c}_M} \right)^2 = \frac{1}{2L} \left( \frac{1-\delta}{1+\delta} \right)^2 \approx \frac{1}{2L},
\end{equation}
because $\delta \ll 1$. This means the DRAN update guarantees a large, stable descent step independent of the original problem's ill-conditioning. In contrast, an unpreconditioned Riemannian gradient step (GD) is bottlenecked by the raw curvature of the problem, where the effective condition number is $\kappa_{\rm GD} = L_H/m_H$. The corresponding descent factor is drastically reduced, i.e.,
\begin{equation}
    \rho_{\rm GD}^* \propto \frac{1}{2L} \left( \frac{1}{\kappa_{\rm GD}} \right)^2 \ll \frac{1}{2L}.
\end{equation}
For ill-conditioned PGO problems, $\kappa_{\rm GD}$ can be massive, forcing $\rho_{\rm GD}^*$ toward zero and causing severe zig-zagging.

Therefore, with preconditioning of the local geometry,  DRAN improves the descent bound by enlarging the effective descent factor $\rho_{\rm AN}$ to enforce $\kappa_{\rm AN} \approx 1$ and thus significantly reducing the number of required iterations and communication rounds compared to first-order methods.
\end{remark}

%% file: experiments_tie.tex
\section{Experiments}\label{sec:experiments}

In this section, we evaluate the performance of our method on PGO problems using (i) benchmark and large-scale simulation datasets and (ii) a real-world multi-robot SLAM dataset. Unless otherwise specified, our optimization algorithms were implemented in Python using the PyPose library~\cite{wang2022pypose}. The benchmark and simulation experiments were executed offline on a workstation with an Intel i7-12800HX CPU and 16\,GB RAM running Ubuntu~20.04. The real-world experiments were executed on a six-robot testbed composed of differential-drive Clearpath Dingo platforms. Each robot runs ROS~1 on an NVIDIA Jetson Orin Nano (8\,GB) with Ubuntu~20.04, and is equipped with an Intel RealSense D435i camera and its built-in BMI055 IMU for VIO, as well as AprilTag-based object identification. Inter-robot communication is realized via a router-based WiFi network using UDP. Ground-truth robot trajectories and object poses are recorded by an OptiTrack motion-capture system, which is used only for evaluation and initial frame alignment. The following section details the datasets and experimental protocols for each setting.

\vspace{-0.25cm}

\begin{table*}[!t]
\centering
\caption{EVALUATION ON PGO BENCHMARK DATASET: Performance comparison of four methods regarding objective values and the Total Transmitted Data (MB), the lower objective and transmitted data volume mean better.
}
\begin{tblr}{
  colspec={c c c c c c c c c | c c c},
  cells = {c},
  cell{1}{1} = {c=2,r=2}{},
  cell{1}{3} = {r=2}{},
  cell{1}{4} = {r=2}{},
  cell{1}{5} = {c=5}{},
  cell{1}{10} = {c=3}{},
  cell{3}{1} = {r=3}{},
  cell{6}{1} = {r=3}{},
  hline{1,9} = {-}{0.08em},
  hline{2} = {5-9,10-12}{},
  hline{3,6} = {-}{},
  hline{4-5,7-8} = {2-12}{},
}
Dataset &                & Vertices & Edges & Objective &       &         &        &                & Total Transmitted Data (MB) &             &              \\
        &                &          &       & F Init.   & F*    & DC2-PGO & DGS    & Ours           & DC2-PGO    & DGS         & Ours         \\
3D      & Parking Garage & 1661     & 6275  & 1.64      & 1.263 & 1.311   & 1.33   & \textbf{1.289} & 17.21         & 5.41          & \textbf{1.23}  \\
        & Sphere         & 2500     & 4949  & 1892      & 1687  & 1687    & 1689   & \textbf{1687}  & 4.95         & 13.73         & \textbf{1.34}  \\
        & Torus          & 5000     & 9048  & 24617     & 24227 & 24227   & 24246  & \textbf{24227} & 17.24         & 9.14         & \textbf{3.23}  \\
2D      & CSAIL          & 1045     & 1171  & 31.50     & 31.47 & 31.47   & 31.49  & \textbf{31.47} & 0.52         & 11.857         & \textbf{0.01}   \\
        & Intel          & 1228     & 1483  & 396.6     & 393.7 & 393.7   & 428.89 & \textbf{393.7} & 4.47        & \textbf{0.06}            &   0.54        \\
        & Manhattan      & 3500     & 5453  & 369.0     & 193.9 & 194.0   & 242.05 & \textbf{194.0} & 115.70        & 152.31        & \textbf{13.47} 
\end{tblr}
\vspace{-0.2cm}
\label{table:pgo}
\end{table*}

\subsection{Performance on Benchmark Datasets}\label{sec:simulation exp}

To evaluate the accuracy and efficiency of our decentralized second-order Riemannian optimization algorithm, we tested it on G2O benchmark pose-graph datasets. Although these datasets contain no objects, they include heterogeneous SLAM graphs with cross-trajectory constraints that can be used to emulate inter-robot consistency requirements. We simulate a multi-robot setting by splitting each trajectory into five segments and assigning each segment to one robot. The physical topology is defined by inter-robot measurements, and the communication topology follows it by exchanging separator variables between connected robots. Our method is fully decentralized and does not require any central node to coordinate update ordering. 
We compared our method with three state-of-the-art baselines: the centralized SE(3) synchronization method SE-Sync~\cite{Rosen19IJRR}, the Distributed Gauss--Seidel (DGS) approach~\cite{choudhary2017distributed}, and DC2-PGO~\cite{tian2021distributed}. All methods optimize the pose graph via convex relaxations of Problem~\ref{problem_formulation2}: SE-Sync and DC2-PGO solve a semidefinite relaxation of the chordal formulation, while DGS solves a distributed quadratic relaxation. For DGS, the successive overrelaxation parameter is set to 1.0 as suggested in~\cite{choudhary2017distributed}. We report the chordal objective values for all methods and the required local iterations $k$ for DGS and DC2-PGO in Table~\ref{table:pgo}, where $k$ equals the number of inter-robot communication rounds. All methods share the same distributed chordal initialization, with the Gauss-Seidel iterations capped at 50 for both rotation and translation to limit initialization communication.


Fig.~\ref{fig:g2o} visualizes the optimization results on the Sphere and Tours datasets obtained by our method. The results show that the five robots successfully reconstruct the overall pose graph by exchanging only separator variables at inter-robot boundaries. As reported in Table~\ref{table:pgo}, our method attains objective values that are very close to the centralized optimum while substantially reducing communication. Specifically, our method matches the best objective value on Sphere, Torus, CSAIL, and Intel, and remains within $2.06\%$ and $0.06\%$ of the best value on Parking Garage and Manhattan, respectively. In terms of transmitted data, our method reduces the total communication volume from $160.09$ MB to $19.82$ MB compared with DC2-PGO, corresponding to an $87.62\%$ reduction. Compared with DGS, the total transmitted data is reduced from $192.51$ MB to $19.82$ MB, corresponding to an $89.70\%$ reduction. On individual datasets, the communication reduction over DC2-PGO ranges from $72.93\%$ to $98.08\%$. These results indicate that DRAN preserves near-centralized estimation accuracy while significantly lowering communication costs in decentralized multi-robot PGO.


\begin{figure}
    \centering
    \includegraphics[width=\columnwidth]{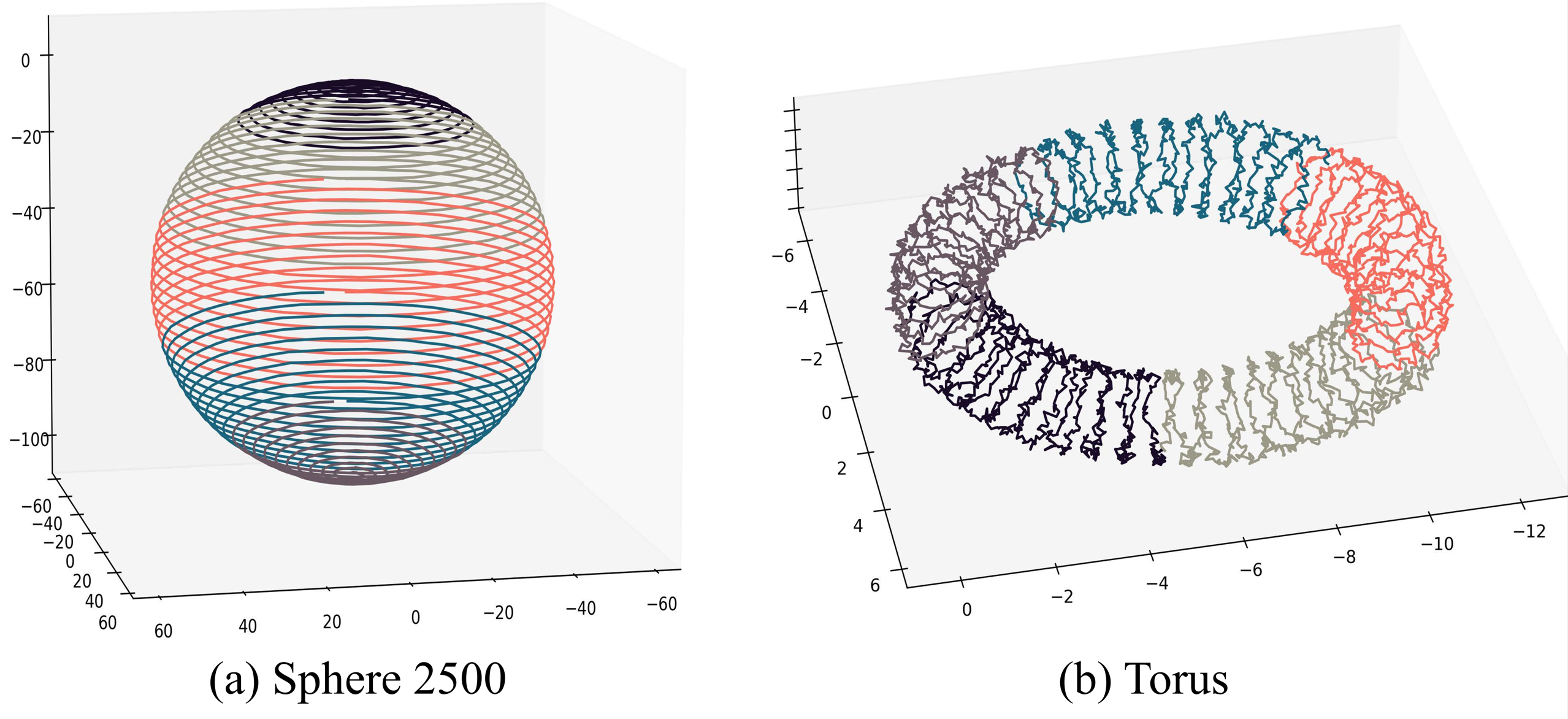}
    \caption{Illustration of two pose graphs optimized using our proposed method. Each graph is distinguished by five colored lines representing the poses of five different robots. (a) corresponds to the Sphere2500 dataset, consisting of 2500 poses and 4949 edges; (b) corresponds to the Torus dataset, consisting of 5000 poses and 9048 edges.}
    \label{fig:g2o}
\end{figure}

\begin{figure}
    \centering
    \includegraphics[width=\columnwidth]{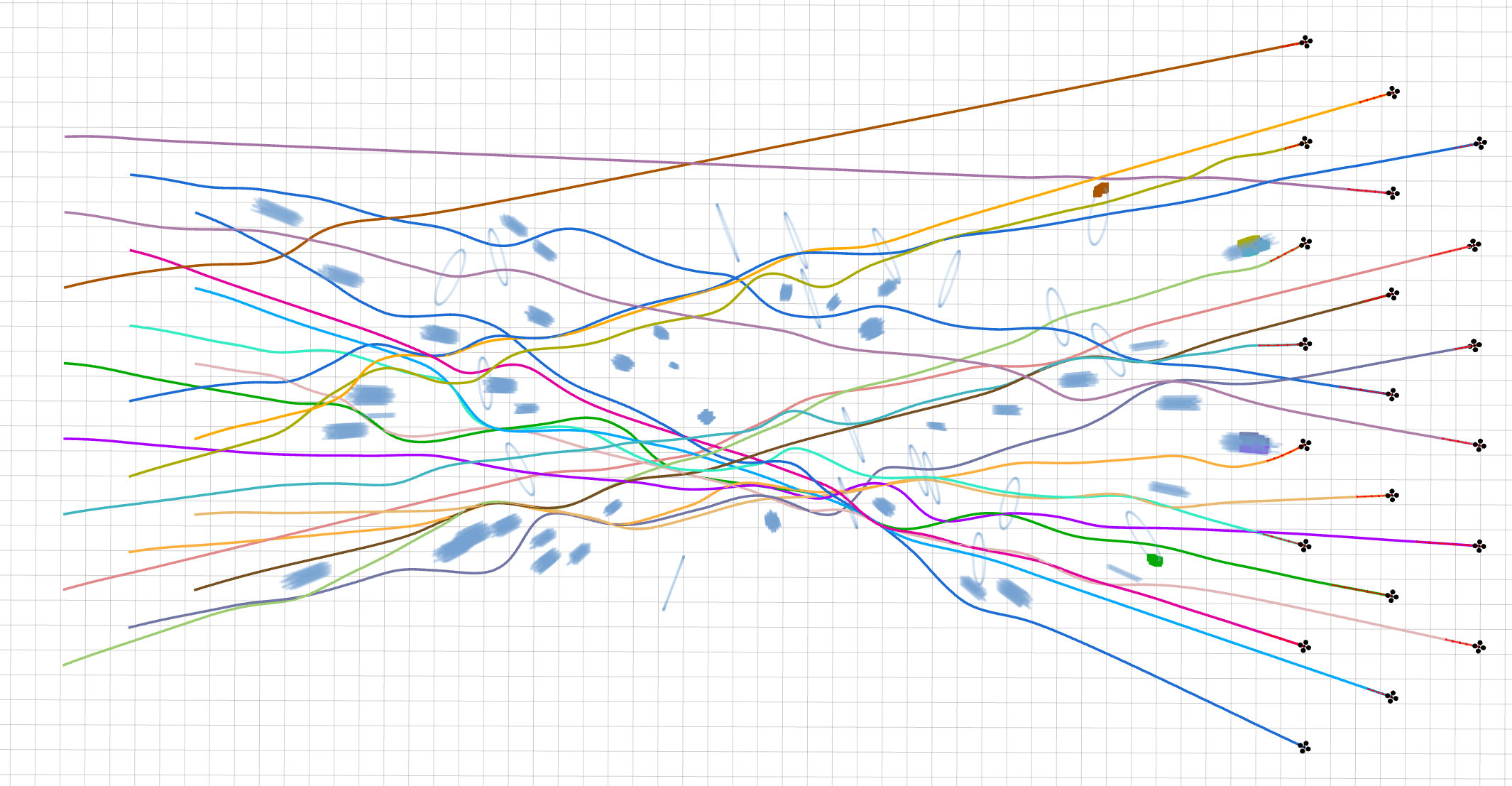}
    \caption{Multi-UAV trajectory estimation and object localization experiments. Colored lines: the trajectories of different UAVs. Star symbols: goal points of the robot. Colored polyhedra: obstacle locations.}
    \label{fig:trajectory20}
\end{figure}

\begin{figure*}
    \centering
    \subfloat[SR-ATE]{
        \label{fig:SR}
        \includegraphics[width=5.6cm]{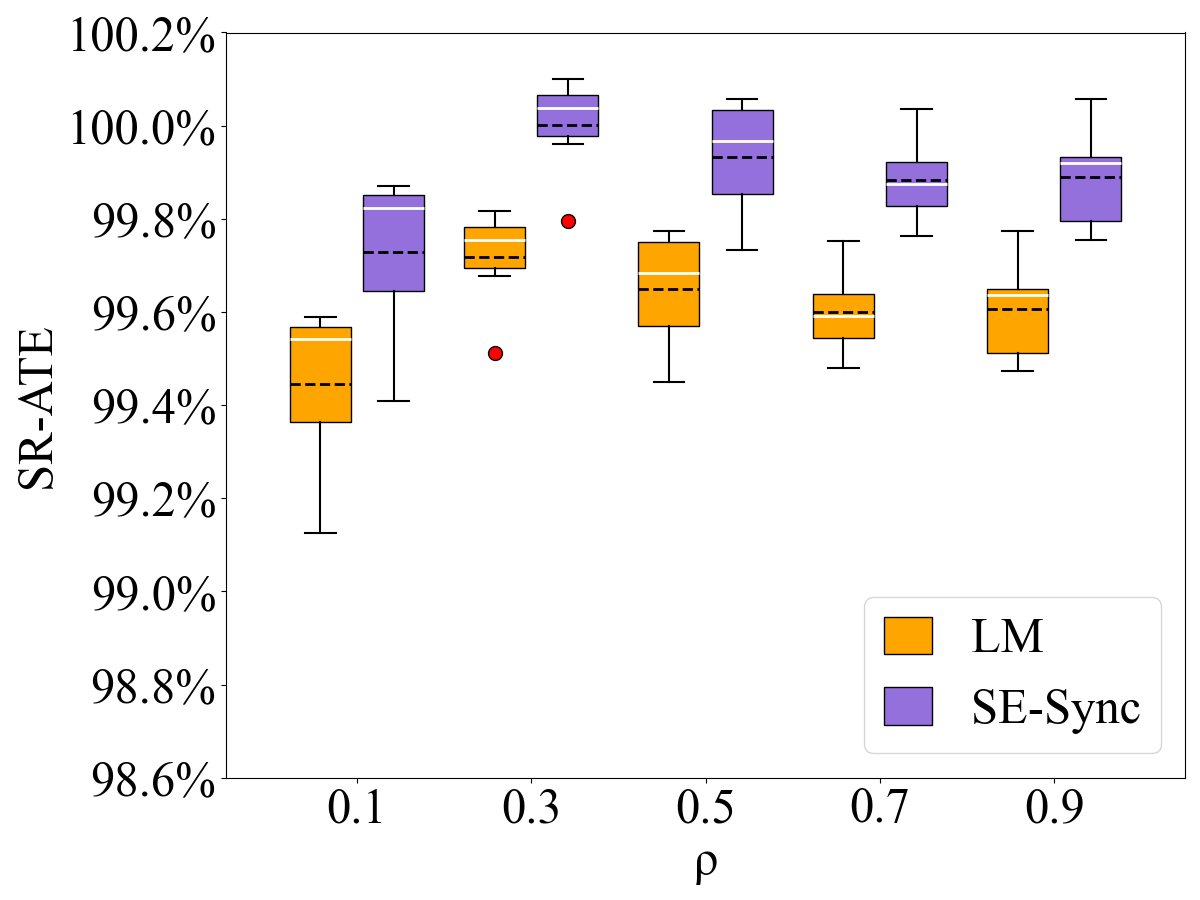}}
    \subfloat[Bias convergence]{
        \label{fig:Bias}
        \includegraphics[width=5.5cm]{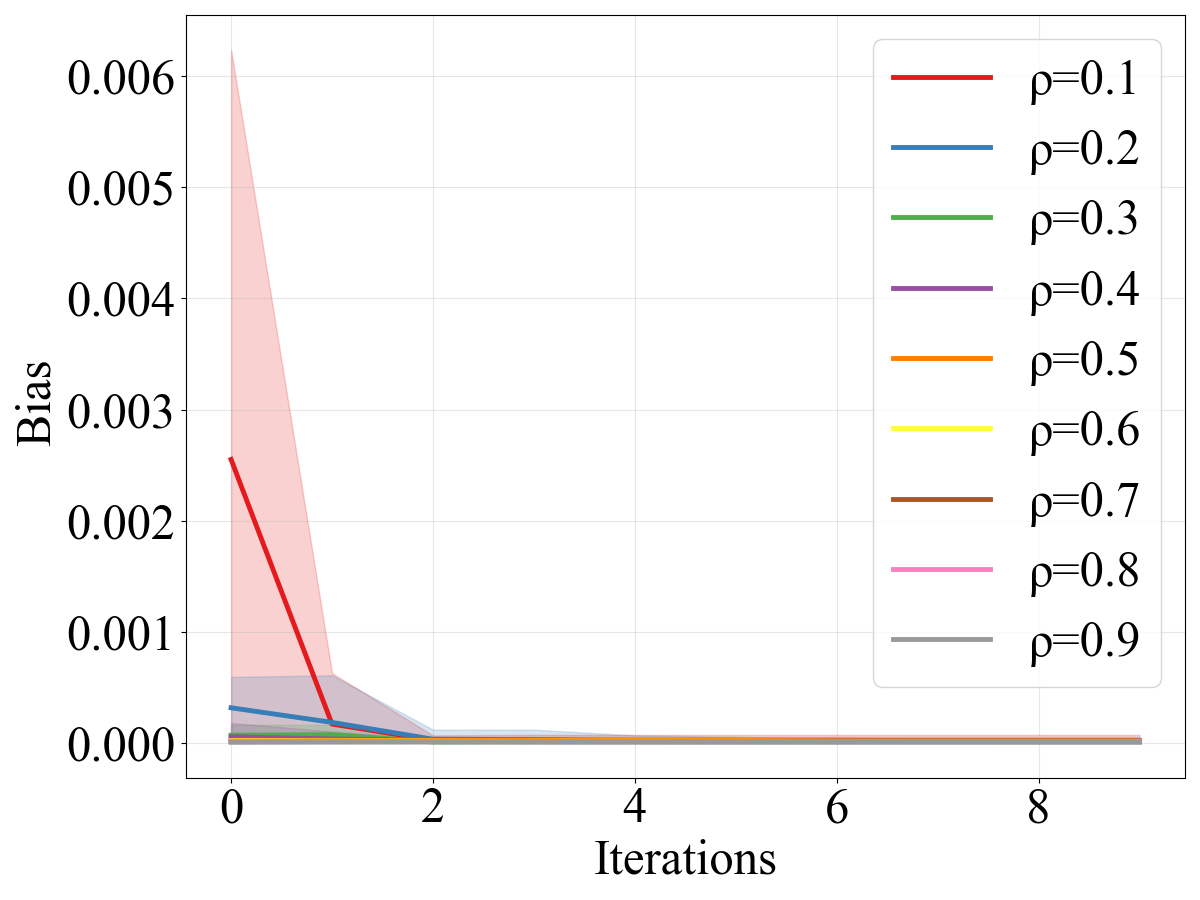}}
    \subfloat[CE convergence]{
        \label{fig:CE}
        \includegraphics[width=5.5cm]{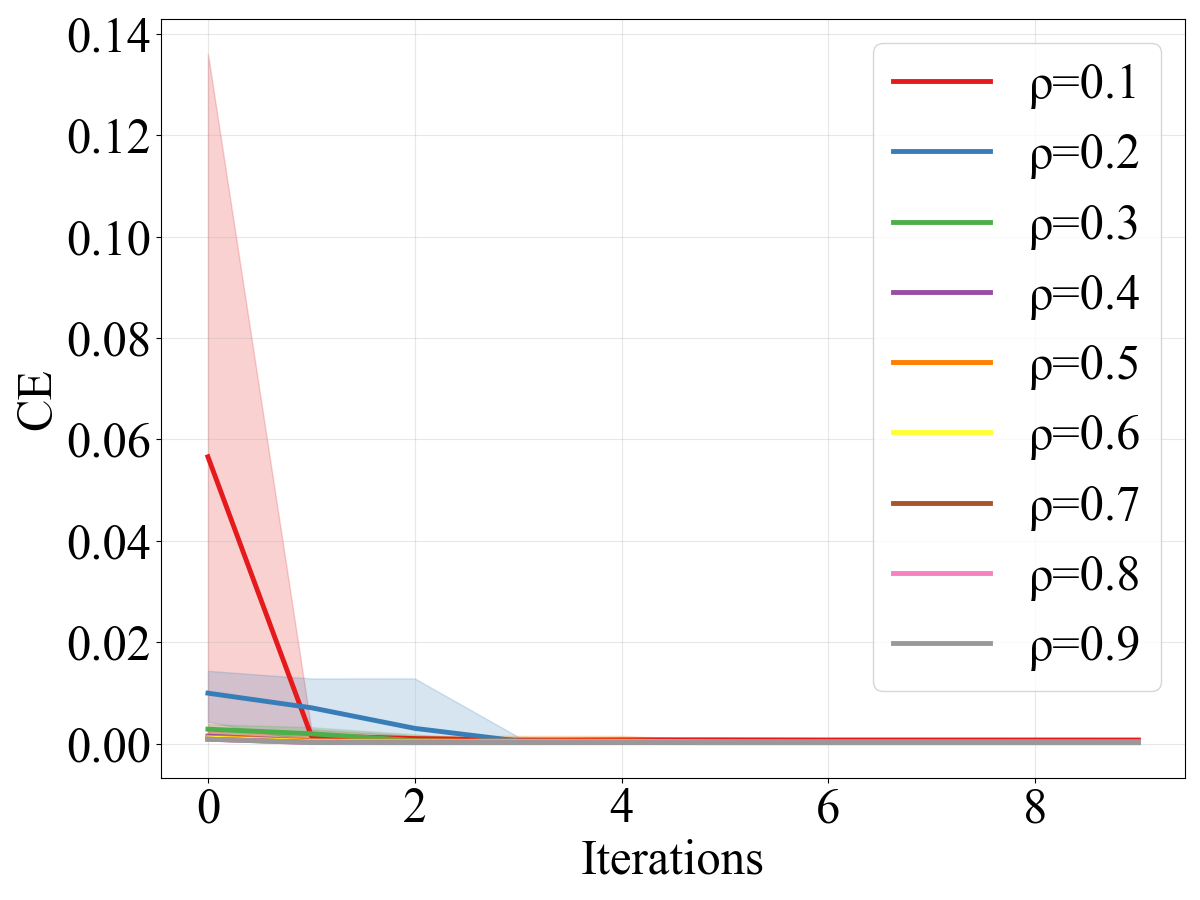}}
    \caption{\label{fig:loss_ego} Performance on large-scale simulation datasets in diverse communication topology.}
\end{figure*}

\vspace{-0.2cm}
\subsection{Performance on Large-scale Simulation Datasets}
\subsubsection{Evaluations of diverse communication topology}
To demonstrate the fully distributed characteristics of our algorithm and verify its flexibility to arbitrarily connected communication topology in a space containing objects, we designed a simulation experiment involving 21 Unmanned Aerial Vehicles (UAV) in a cluttered space based on platform EGO-Swarm \cite{ego-swarm}, as shown in Fig.~\ref{fig:trajectory20}. The simulation space contains 60 obstacles with randomly assigned poses, and 21 UAVs passed through obstacles and reached specified locations. Without loss of generality, we assume that the world-to-robot frame transform is known at the initial stage. Each UAV will get the initial pose of its trajectory using odometry measurements and obtain the corresponding noisy observations when its distance to any obstacle gets below $ 5 m$. All measurements are corrupted by Langevin rotation noise with $1^{\circ}$ standard deviation and Gaussian translation noise with $0.005 m$ standard deviation. Specifically, translations and rotations of the observations are randomly sampled as follows, $\tilde {t}_{ij}= t_{ij} + t_{ij}^{\epsilon},\ t_{ij}^{\epsilon} \sim \mathcal{N} (0,\tau_{ij}I_d)$, $\tilde {R}_{ij} = R_{ij}^{\epsilon}R_{ij},\ R_{ij}^{\epsilon} \sim Langevin (I_d, \kappa_{ij})$, 
where $\kappa_{ij}$ and $\tau_{ij}$ denote the root-mean-squared (RMS) error for rotational ($\tilde {R}_{ij}$) and translational ($\tilde {t}_{ij}$) measurements, respectively.

The degree of communication connectivity $\rho$ is defined as the probability of generating communication edges between any two robots. To compare the effectiveness of our method under various levels of communication connectivity, we collect the ground truth poses of the UAVs and obstacles in the simulation environment. We use SE-Sync\cite{Rosen19IJRR} and the Levenberg-Marquardt (LM) method from the G2O library, two centralized methods, as baselines for comparison when solving Problem~\ref{problem_formulation2}. 

We use the success ratio of ATE ($SR_{ATE}$)  as a criterion for evaluating the results, which is defined as follows:
\begin{equation}
    SR_{ATE}=\frac{ATE^0-ATE^k}{ATE^0-ATE^*}\times 100\%,
\end{equation}
where $\bar{T}_{i}$ is the ground truth value of pose $i$, $ATE=\sqrt{\frac{1}{N}\sum_{i=1}^N{\frac{1}{n_i}\sum_{j=1}^{n_i}{\lVert \text{Log}_I\left( \bar{x}_{i_j}^{-1}x_{i_j} \right) \rVert ^2}}}$, $ATE^0$ is the ATE of the initial poses, $ATE^k$ is the ATE of the result poses of our method, and $ATE^*$ is the ATE of the result poses of the centralized method. Besides, we defined the success rate of cost in terms of the cost function value as follows,
\begin{equation}
    SR_{cost}=\frac{F_{init}-F^k}{F_{init}-F^*}\times 100\%,
\end{equation}
which characterizes the gap between the cost function value achieved by our decentralized method and the cost function value achieved by a centralized algorithm. 

To measure the consistency and accuracy of the estimation results of all robots regarding the poses of environmental obstacles, we define the consistency error (CE) as follows, 
\begin{equation}
       CE=\frac{1}{N}\sum_{i=1}^N{\frac{1}{m}\sum_{j=1}^m{\lVert \text{Log}_I\left( y_{i_j}^{-1}y_{avg_{\text{j}}} \right) \rVert ^2}},   
\end{equation}
where $\boldsymbol{y}_{avg}=\underset{\left[ y_1,\cdots ,y_m \right]}{arg\min}\frac{1}{N}\sum_{i=1}^N{\frac{1}{m}\sum_{j=1}^m{\lVert \text{Log}_I\left( y_{i_j}^{-1}y_j \right) \rVert ^2}}$ is the average of all robot estimates of the obstacles' poses.
Moreover, we introduce a metric noted as Bias that reflects the error between the average and ground truth value:
\begin{equation}
    Bias=\sqrt{\frac{1}{N}\sum_{i=1}^N{\frac{1}{m}\sum_{j=1}^{m}{\lVert \text{Log}_I\left( \bar{y}_{i_j}^{-1}y_{i_j} \right) \rVert ^2}}}.
\end{equation}


We evaluated the success rate (defined as the optimality gap of cost) under varying levels of communication connectivity. The results presented in Fig.~\ref{fig:loss_ego} and Fig.~\ref{fig:SR_cost_matlab_respone} show the quantitative results of the large-scale simulation. Our method maintains a high success ratio in scenarios with different communication topologies. It demonstrates resilience to the sparseness of communication topology, showcasing its adaptability in unstructured and complex environments. Moreover, we find the estimates of all robots for the shared variables (the poses of environmental obstacles) converge to the consensus rapidly with a few iterations. At the same time, the mean value rapidly approaches the ground truth value. Besides, we found that in a few tests, the results of our decentralized algorithm are closer to the ground truth for different topology degrees than the centralized solution results. This is mainly because the PGO problem is non-convex, and the centralized method cannot find the optimal solution easily.

\begin{figure}
    \centering
    \includegraphics[width=\columnwidth]{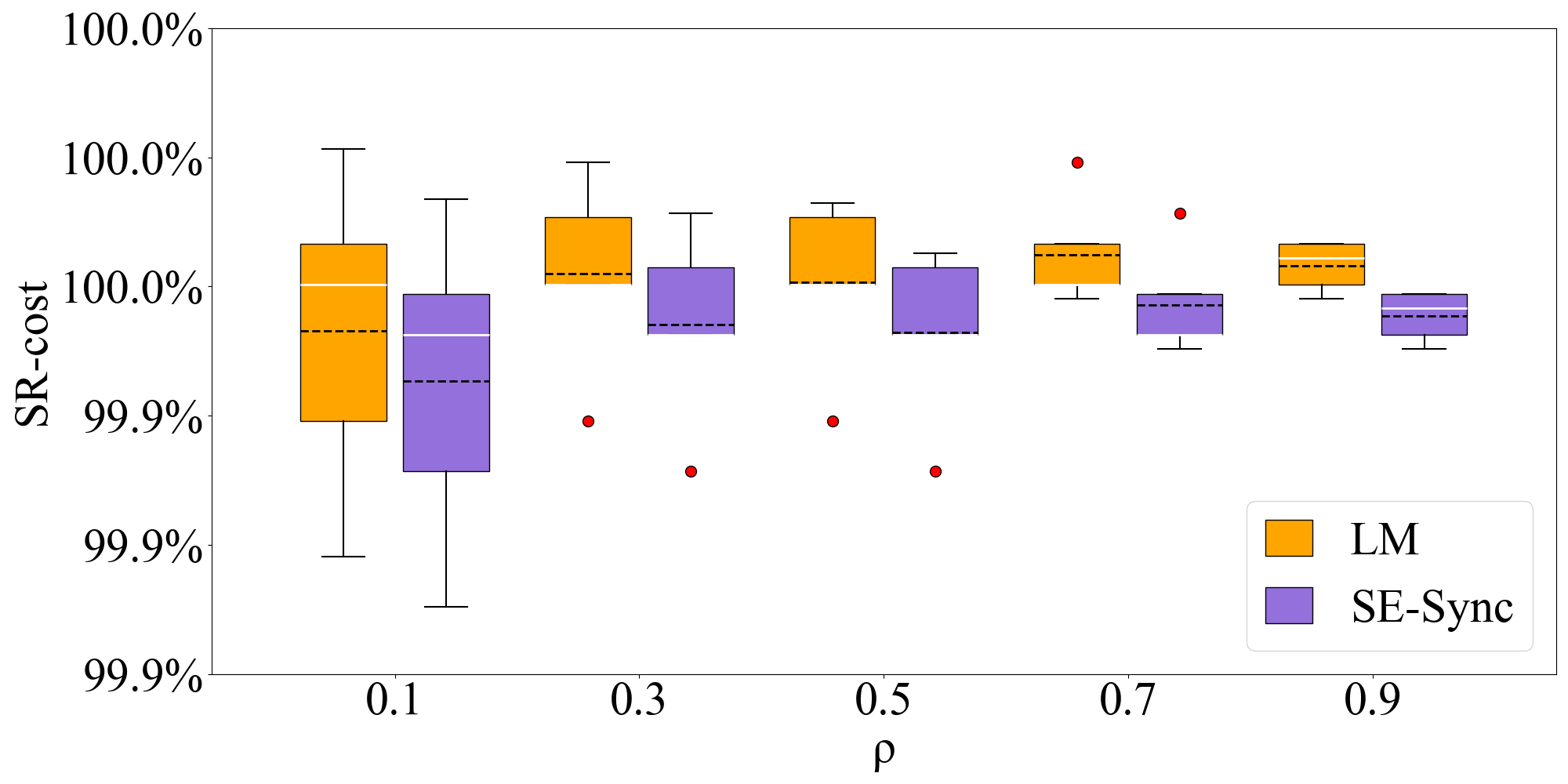}
    \caption{The success rate of cost on large-scale simulation datasets in diverse communication topologies.}
    \label{fig:SR_cost_matlab_respone}
\end{figure}


    

\subsubsection{Evaluations of diverse measurement noises}
In this experiment, we evaluate the robustness of our distributed PGO method under different levels of noise in measurements. Noise is introduced separately to both translational and rotational measurements with increasing magnitudes, with noise levels noted by $\tau$ for translation and $\kappa$ for rotation. The system consists of 21 robots, each exchanging information within a predefined communication network with $\rho=0.3$. We compare the performance of our distributed approach to two centralized PGO methods, LM and SE-Sync\cite{Rosen19IJRR}.

\begin{figure}
    \centering
    \subfloat[$\kappa=0.01$]{
        \label{fig:SR with Tau}
        \includegraphics[width=\columnwidth]{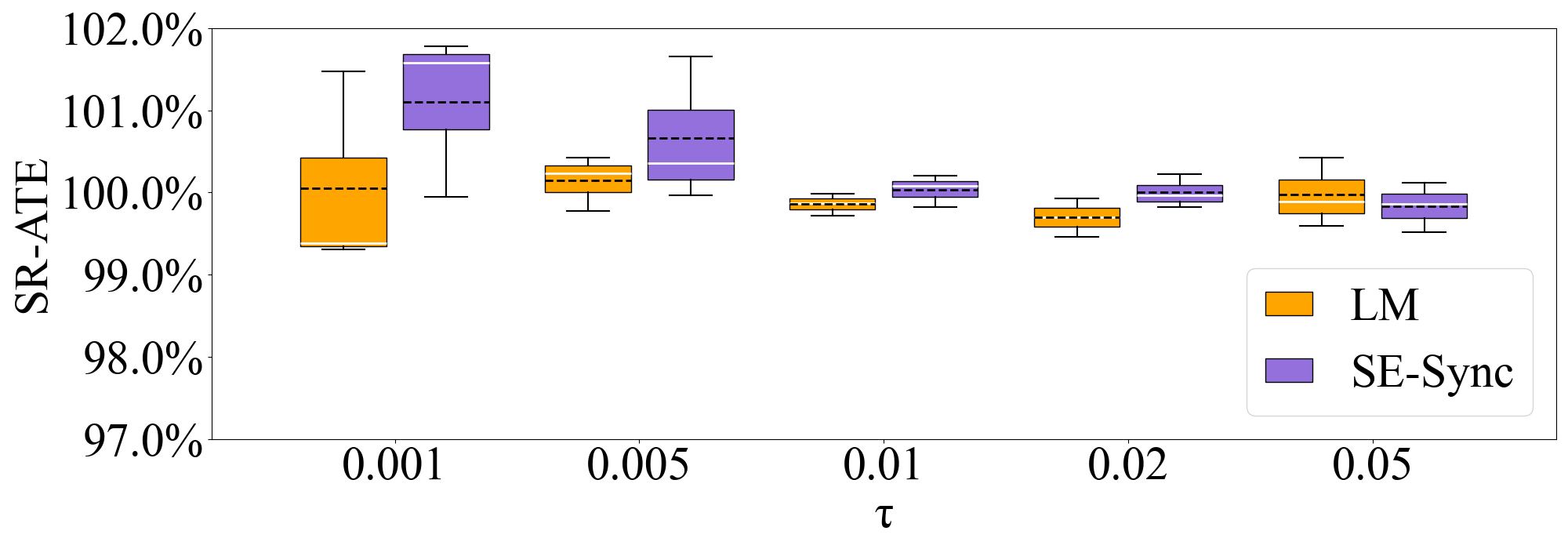}}
    \vspace{-0.3cm}
    \subfloat[$\tau=0.001$]{
        \label{fig:SR with kappa}
        \includegraphics[width=\columnwidth]{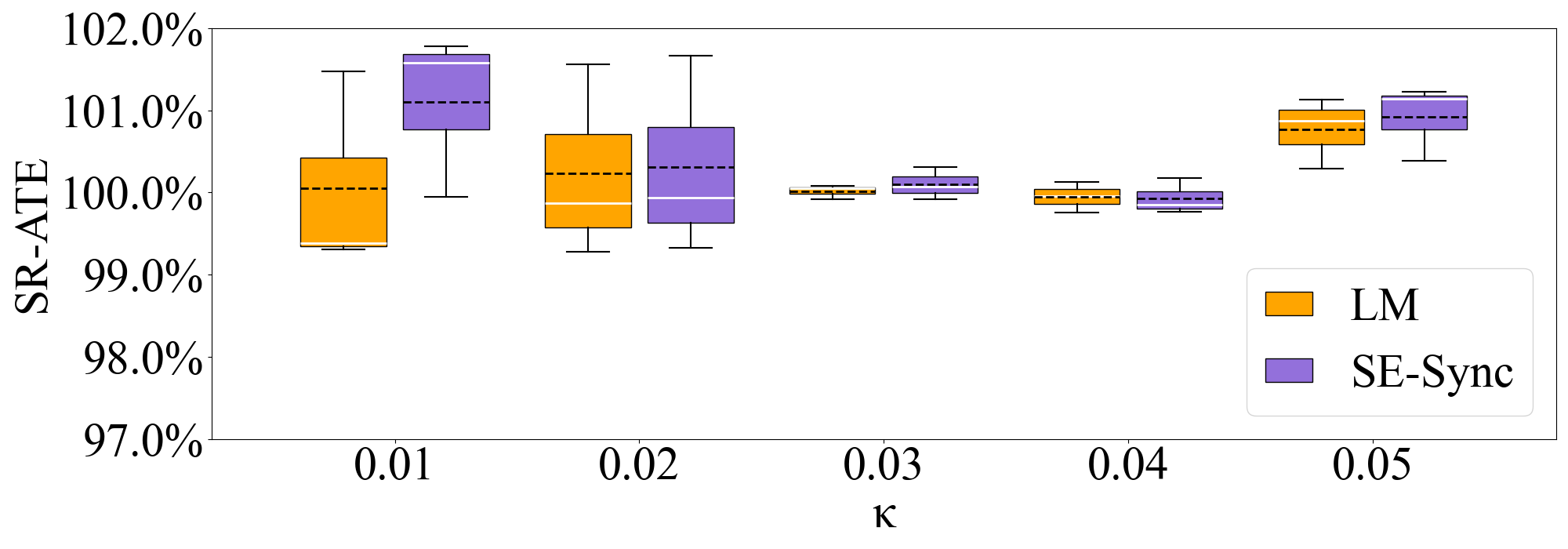}}
    \caption{\label{fig:loss_tb} Success ratio of ATE from the multi-level noisy datasets.}
\end{figure}

\begin{figure}
    \centering
    \subfloat[$\kappa=0.01$]{
        \label{fig:SR ATE with Tau}
        \includegraphics[width=\columnwidth]{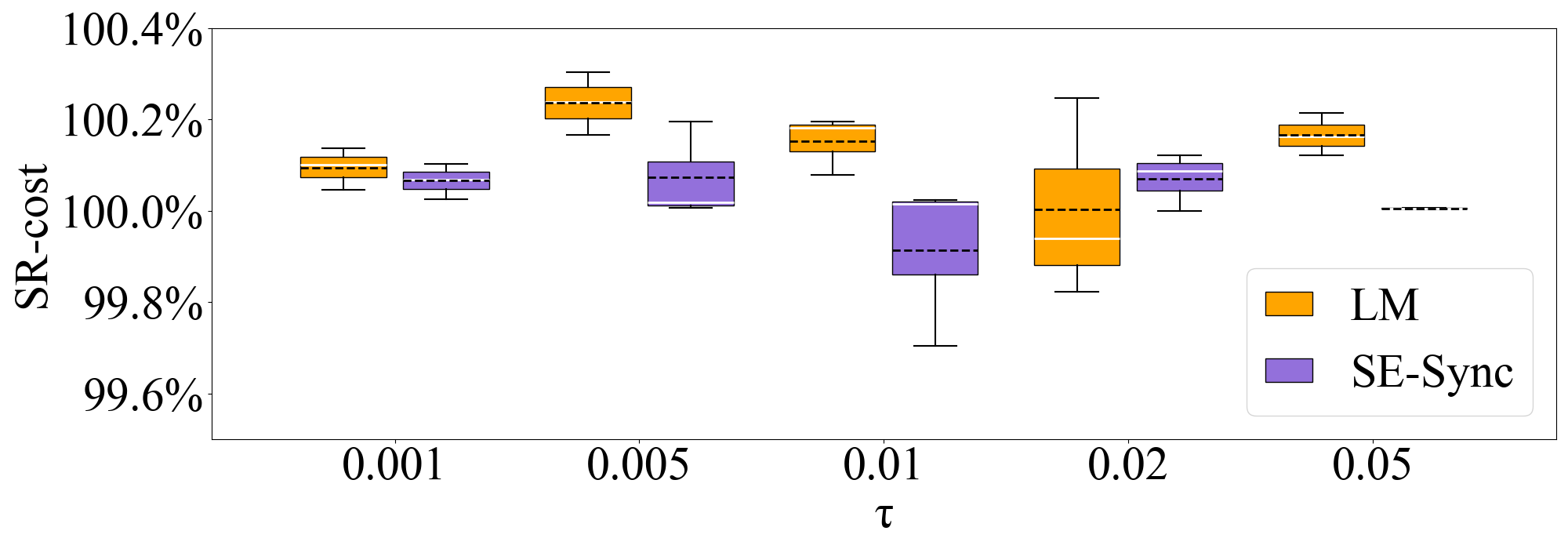}}
    \hspace{-0.1cm}
    \subfloat[$\tau=0.001$]{
        \label{fig:SR ATE with kappa}
        \includegraphics[width=\columnwidth]{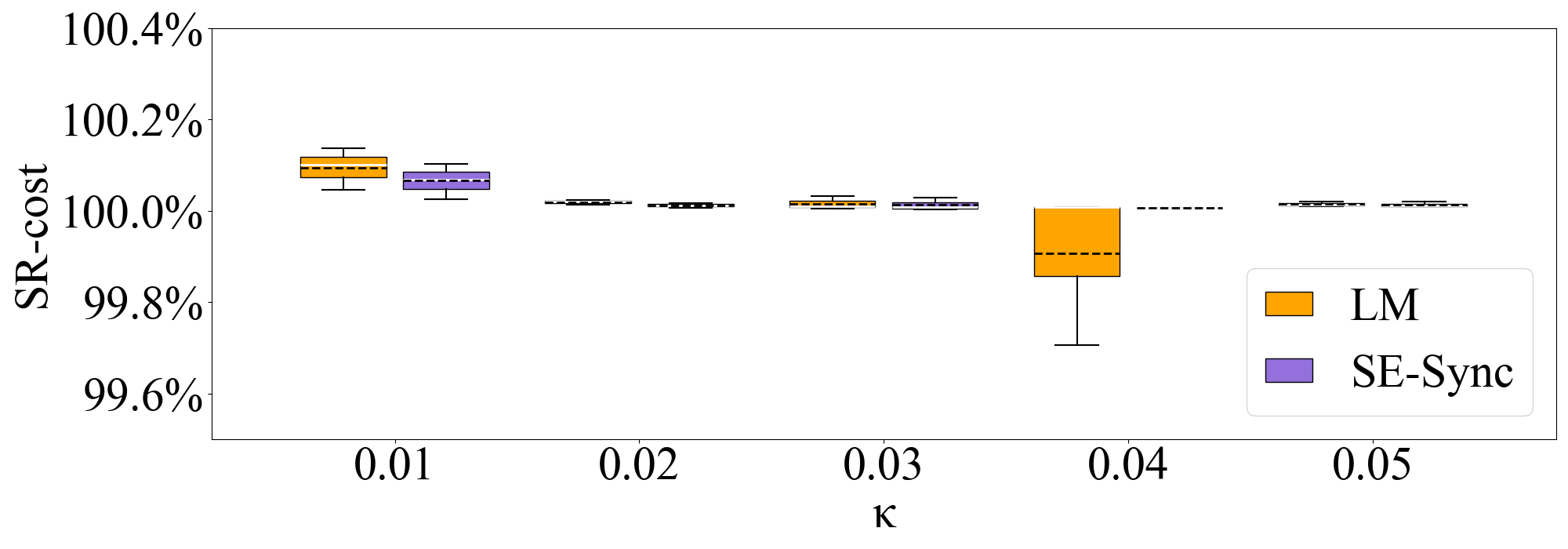}}
    \caption{\label{fig:OV_tb_response} Success ratio of cost from the multi-level noisy datasets.}

\end{figure}

The success ratio of our method under a multi-level noisy dataset are shown in Fig.~\ref{fig:loss_tb} and Fig.~\ref{fig:OV_tb_response}, demonstrating our distributed method achieves performance comparable to the centralized baseline, particularly at lower noise levels. As noise increases, the distributed method maintains robustness, showing only a marginal increase in error compared to the centralized approach. This indicates that our method effectively mitigates the impact of noise, maintaining consistency in both translational and rotational estimations, making it more suitable for noisy observations in real-world applications.

\subsubsection{Evaluations of the effectiveness of Schur complement}
To validate the effectiveness of using the Schur complement in our method, we performed several numerical experiments focusing on both computation time and accuracy. The experiment was conducted on a predefined communication network with a network density of $\rho = 0.5$, and the results were averaged over 20 trials for each configuration. The comparison was made between our method with the Schur complement (DRAN w/ SC) and without it (DRAN w/o SC). A summary of the results for the large-scale simulation dataset (Ego) is presented in Table~\ref{exp:time}. 

\begin{table}
\caption{DRAN performance comparison with or without Schur complement in Ego dataset.}
\centering
\footnotesize
\resizebox{.98\columnwidth}{!}{%
\begin{tblr}{
  cells = {c},
  cell{2}{1} = {r=2}{},
  cell{2}{2} = {r=2}{},
  hline{1-2,4} = {-}{},
  hline{3} = {3-5}{},
}
\textbf{DataSet} & \textbf{Private/public} & \textbf{Method} & \textbf{Time ave (ms)} & \textbf{Cost}   \\
Ego              & 8.735                   & DRAN w/ SC      & \textbf{31.39}         & \textbf{0.0919} \\
                 &                         & DRAN w/o SC     & 53.04                  & 0.0920          
\end{tblr}
}
\label{exp:time}
\end{table}

The experimental results illustrate that the Schur complement significantly reduces the average computation time without compromising the accuracy of the final optimization. Specifically, the method with the Schur complement demonstrated a reduction in computation time by approximately $41\%$ compared to the method without the Schur complement. On the Ego dataset, the average computation time with the Schur complement was $31.39 ms$, in contrast to $53.04 ms$ for the method without the Schur complement.
In terms of accuracy, the final optimized pose showed near-identical results: an objective value of $0.0919$ with the Schur complement and $0.0920$ without it. These results confirm that the application of the Schur complement improves computational efficiency without sacrificing optimization quality.


\subsection{Real-World Multi-Robot System Validation}\label{sec:physical exp}
To validate deployability under real sensing and communication constraints, we collected a real-world multi-robot dataset in an indoor environment at the Yuquan Campus, Zhejiang University, Hangzhou, China (Fig.~\ref{fig:dingo_system}). We placed 9 AprilTag-tagged boxes in the scene to provide object identities for cross-robot data association. During data collection, the robots were teleoperated to execute approximately circular trajectories, producing repeated observations and sufficient inter-robot overlap of observation region. For evaluation, ground-truth robot trajectories and object poses were recorded by an OptiTrack motion-capture system. The motion capture system is used only for evaluation and for initializing the world-to-robot frame transform at the start of each run, and is not involved in the subsequent estimation. Each robot obtains odometry by running VINS-Mono~\cite{vinsmono} on the onboard camera (RGB stream, 640$\times$480P at 30~FPS) stream and BMI055 IMU measurements (100~Hz), while loop closures are detected by DBoW2~\cite{DBoW2} and incorporated as relative-pose constraints. To ensure reproducible sparsity, we enforce a predefined random connected sparse communication graph $G_c$ at the application layer: although all robots are IP-reachable on the same WLAN, robot $i$ only exchanges messages with its neighbor set $\mathcal{N}_i$ specified by $G_c$. Communication is rate-limited to at most 2~Hz, and each message contains only the variables required by Algorithm~\ref{alg:NT1} (separator states and object pose estimates). On average, each robot's pose graph contains 112 nodes and 165 edges. The initial trajectory and object pose estimates are shown in Fig.~\ref{fig:dingo_initial}.

\begin{figure}
    \centering
    \includegraphics[width=\columnwidth]{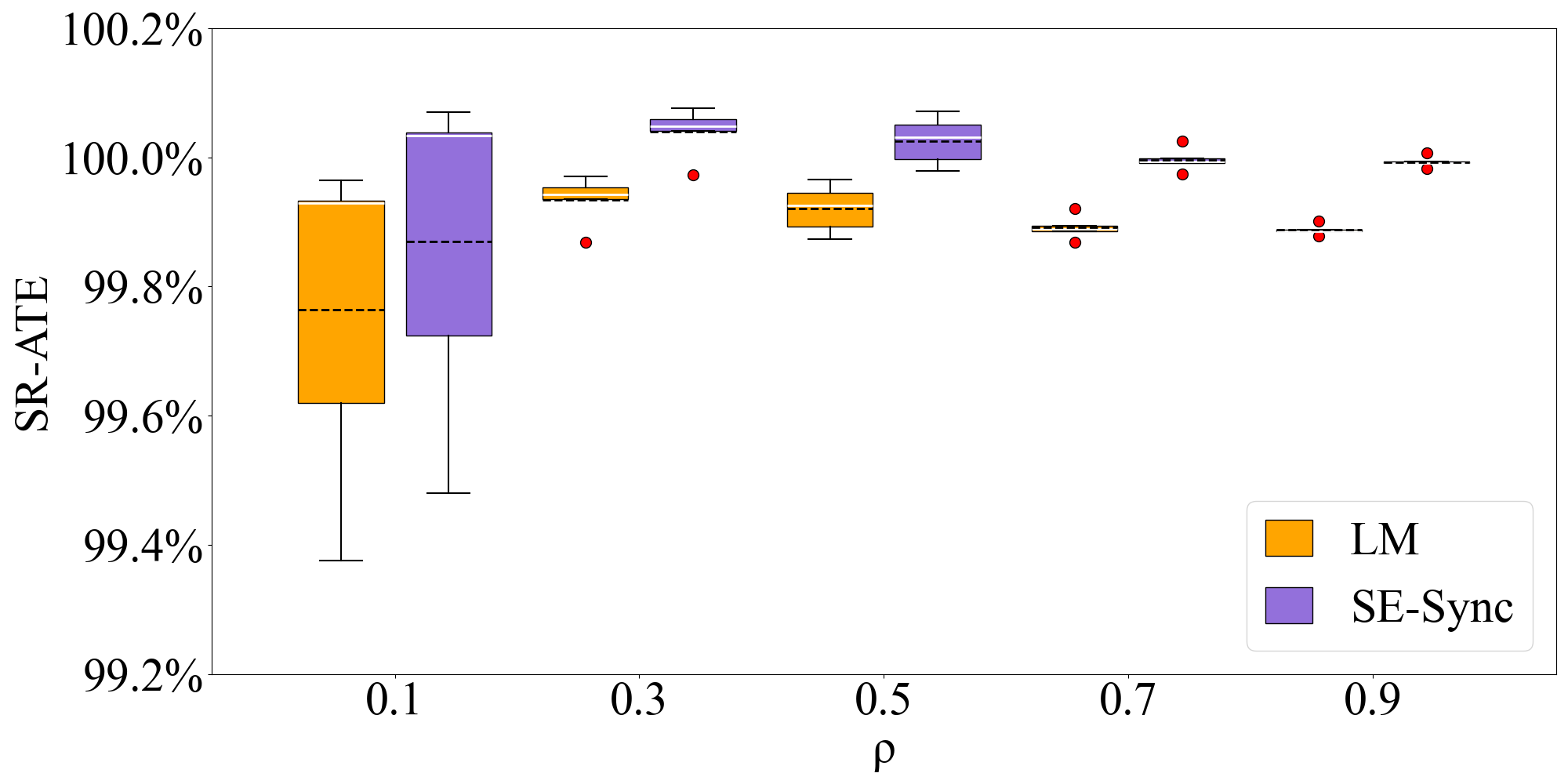}
    \caption{The success ratio of ATE under different degrees of connectivity for real-world dataset experiments.}
    \label{fig:dingo_exp}
\end{figure}

As the optimized trajectories and object poses shown in Fig. \ref{fig:dingo_optimized}, the algorithm can maintain consistent and accurate pose estimates, closely matching the ground truth obtained from the motion capture system. Similar to the simulation experiments, the quantitative metrics specific to the success ratio are reported in Fig.~\ref{fig:dingo_exp}. Our distributed PGO method demonstrated a high success ratio across various communication topologies, with performance close to the centralized baseline. Furthermore, the results confirmed the method's adaptability to different communication structures, showing robustness in communication-poor scenarios. This highlights the advantage of our approach in practical multi-robot applications, where communication constraints frequently arise. Notably, our method achieved similar accuracy to the centralized approach, even in sparse networks.

%% file: appendix.tex
\appendix
\section{Proofs for the Convergence Analysis}
\label{app:convergence_proofs}

This appendix collects the detailed proofs omitted from Section~\ref{sec:convergence}.

\subsection{Proof of Lemma~\ref{lem:conv_lm_spd} (Positive definiteness of the LM approximation)}
\label{app:proof:lem:conv-lm-spd}

For any tangent vector $\xi\in T_{z_i^k}\mathcal M_i$,
\[
    \langle \xi,M_i^k\xi\rangle
    =
    \|(J_i^k)\xi\|^2
    +
    \mu_i^k\|\xi\|^2
    \ge
    \underline\mu\|\xi\|^2 .
\]
Thus $M_i^k\succeq \underline\mu I$.
Moreover,
\[
    \langle \xi,M_i^k\xi\rangle
    =
    \|(J_i^k)\xi\|^2
    +
    \mu_i^k\|\xi\|^2
    \le
    (\bar J^2+\bar\mu)\|\xi\|^2 .
\]
Hence $M_i^k\preceq (\bar J^2+\bar\mu)I$.
The proof is complete.

\subsection{Proof of Lemma~\ref{lem:conv_schur_spd} (Positive definiteness of the Schur complement)}
\label{app:proof:lem:conv-schur-spd}

Since $M_i^k\succeq c_m I\succ0$, every principal block of $M_i^k$ is positive definite. 
Therefore $A_i^k\succ0$.

Because $M_i^k\succ0$ and $A_i^k\succ0$, the Schur complement 
\[
    \hat H_i^k=B_i^k-(C_i^k)^\top(A_i^k)^{-1}C_i^k
\]
is positive definite.

It remains to prove the explicit spectral bounds. 
For the lower bound, use the block inverse identity:
the lower-right block of $(M_i^k)^{-1}$ is $(\hat H_i^k)^{-1}$. 
Since $M_i^k\succeq c_m I$, we have
\[
    (M_i^k)^{-1}\preceq \frac{1}{c_m}I .
\]
Therefore,
\[
    (\hat H_i^k)^{-1}\preceq \frac{1}{c_m}I,
\]
which implies $\hat H_i^k\succeq c_m I .$

For the upper bound, for any public tangent vector $v$,

\begin{equation*}
    \begin{aligned}
        \langle v,\hat H_i^k v\rangle
        &=
        \min_{u}
        \left\langle 
        \begin{bmatrix}
            u\\v
        \end{bmatrix},
        M_i^k
        \begin{bmatrix}
            u\\v
        \end{bmatrix}
        \right\rangle\\
        &\le
        \left\langle 
        \begin{bmatrix}
            0\\v
        \end{bmatrix},
        M_i^k
        \begin{bmatrix}
            0\\v
        \end{bmatrix}
        \right\rangle
        \le
        c_M\|v\|^2 .
    \end{aligned}
\end{equation*}
Thus $\hat H_i^k\preceq c_M I$.
The proof is complete.

\subsection{Proof of Theorem~\ref{thm:conv_descent} (Curvature-preconditioned descent)}
\label{app:proof:thm:conv-descent}

By Assumption~\ref{ass:conv_pullback_smooth}, with 
$\xi=\alpha\xi_{i,N}^k$, we have
\[
\begin{aligned}
    \hat L_i(z_i^{k+1})
    &=
    \hat L_i\bigl(\operatorname{Retr}_{z_i^k}(\alpha\xi_{i,N}^k)\bigr)\\
    &\le
    \hat L_i(z_i^k)
    +
    \alpha\langle g_i^k,\xi_{i,N}^k\rangle
    +
    \frac{L\alpha^2}{2}
    \|\xi_{i,N}^k\|^2 .
\end{aligned}
\]
Since $\xi_{i,N}^k=-(M_i^k)^{-1}g_i^k$,
\[
    \langle g_i^k,\xi_{i,N}^k\rangle
    =
    -\langle g_i^k,(M_i^k)^{-1}g_i^k\rangle .
\]
Using $M_i^k\preceq c_M I$, we obtain
\[
    \langle g_i^k,(M_i^k)^{-1}g_i^k\rangle
    \ge
    \frac{1}{c_M}\|g_i^k\|^2 .
\]
Using $M_i^k\succeq c_m I$, we also have
\[
    \|\xi_{i,N}^k\|
    =
    \|(M_i^k)^{-1}g_i^k\|
    \le
    \frac{1}{c_m}\|g_i^k\| .
\]
Therefore,
\[
\begin{aligned}
    \hat L_i(z_i^{k+1})
    &\le
    \hat L_i(z_i^k)
    -
    \frac{\alpha}{c_M}\|g_i^k\|^2
    +
    \frac{L\alpha^2}{2c_m^2}\|g_i^k\|^2\\
    &=
    \hat L_i(z_i^k)
    -
    \left(
    \frac{\alpha}{c_M}
    -
    \frac{L\alpha^2}{2c_m^2}
    \right)
    \|g_i^k\|^2 .
\end{aligned}
\]
The stepsize condition \eqref{eq:conv_stepsize} guarantees $\rho_{\rm AN}>0$.
The proof is complete.

\subsection{Proof of Corollary~\ref{cor:conv_stationarity} (Non-asymptotic first-order stationarity)}
\label{app:proof:cor:conv-stationarity}

Summing \eqref{eq:conv_descent_bound} from $k=0$ to $K-1$ gives
\[
    \rho_{\rm AN}
    \sum_{k=0}^{K-1}
    \|g_i^k\|^2
    \le
    \hat L_i(z_i^0)-\hat L_i(z_i^K)
    \le
    \hat L_i(z_i^0)-\hat L_i^\star .
\]
Therefore,
\[
    \min_{0\le k\le K-1}\|g_i^k\|^2
    \le
    \frac{1}{K}
    \sum_{k=0}^{K-1}\|g_i^k\|^2
    \le
    \frac{
        \hat L_i(z_i^0)-\hat L_i^\star
    }{
        \rho_{\rm AN}K
    } .
\]
The proof is complete.

\subsection{Proof of Lemma~\ref{lem:lm_spectral_equivalence} (Local spectral equivalence of the LM approximation)}
\label{app:proof:lem:lm-spectral-equivalence}

From \eqref{eq:true_hessian_decomposition} and \eqref{eq:lm_approx_again}, we have
\[
    M_i^k-H_i^k
    =
    \mu_i^k I
    -
    R_i^k .
\]
Therefore,
\[
    \|M_i^k-H_i^k\|
    \le
    \mu_i^k+\|R_i^k\|
    \le
    \varepsilon_\mu^k+\varepsilon_H^k
    \le
    \delta m_H .
\]
Hence,
\[
    -\delta m_H I
    \preceq
    M_i^k-H_i^k
    \preceq
    \delta m_H I .
\]
Since Assumption~\ref{ass:local_regular_gauge_fixed} gives $H_i^k\succeq m_H I$, we have
\[
    \delta m_H I\preceq \delta H_i^k .
\]
Thus,
\[
    -\delta H_i^k
    \preceq
    M_i^k-H_i^k
    \preceq
    \delta H_i^k .
\]
Rearranging the above inequality gives
\[
    (1-\delta)H_i^k
    \preceq
    M_i^k
    \preceq
    (1+\delta)H_i^k .
\]
The proof is complete.

\subsection{Proof of Theorem~\ref{thm:conv_conditioning} (Conditioning effect of approximate Newton preconditioning)}
\label{app:proof:thm:conv-conditioning}

From \eqref{eq:lm_spectral_equivalence}, we have
\[
    M_i^k
    \preceq
    (1+\delta)H_i^k .
\]
Multiplying both sides by $(M_i^k)^{-1/2}$ from the left and right yields
\[
    I
    \preceq
    (1+\delta)
    (M_i^k)^{-1/2}H_i^k(M_i^k)^{-1/2}.
\]
Therefore,
\[
    \frac{1}{1+\delta}I
    \preceq
    (M_i^k)^{-1/2}H_i^k(M_i^k)^{-1/2}.
\]

Similarly, from
\[
    M_i^k
    \succeq
    (1-\delta)H_i^k ,
\]
we obtain
\[
    I
    \succeq
    (1-\delta)
    (M_i^k)^{-1/2}H_i^k(M_i^k)^{-1/2},
\]
which gives
\[
    (M_i^k)^{-1/2}H_i^k(M_i^k)^{-1/2}
    \preceq
    \frac{1}{1-\delta}I .
\]
Combining the two inequalities proves \eqref{eq:conv_preconditioned_spectrum}. 
Thus, all eigenvalues of the preconditioned Hessian lie in
\[
    \left[
        \frac{1}{1+\delta},
        \frac{1}{1-\delta}
    \right],
\]
and its condition number satisfies
\[
    \kappa_{\rm AN}
    \le
    \frac{1/(1-\delta)}{1/(1+\delta)}
    =
    \frac{1+\delta}{1-\delta}.
\]
On the other hand, without preconditioning, the relevant local curvature matrix is $H_i^k$, whose eigenvalues lie in $[m_H,L_H]$. 
Hence, the corresponding condition number is $\kappa_{\rm GD}=L_H/m_H$. 
The proof is complete.